\documentclass{article}

\usepackage{microtype}
\usepackage{graphicx}
\usepackage{subfigure}
\usepackage{booktabs} % for professional tables
\usepackage{amsmath}

\usepackage{kz_style}

\usepackage{array}
\newcolumntype{M}[1]{>{\centering\arraybackslash}p{#1}}

\newenvironment{emphenv}
{\itshape} % 开始环境时应用斜体
{}        % 结束环境时无需特殊处理

\usepackage{hyperref}

\usepackage[accepted]{icml2024}

\icmltitlerunning{Exploring the LLM Journey from Cognition to Expression with Linear Representations}

\begin{document}

\twocolumn[
% \icmltitle{Aligning Thought and Word:
% Exploring the LLM Journey from Cognition to Expression with Linear Representation}

\icmltitle{Exploring the LLM Journey from Cognition to Expression with Linear Representations}

% It is OKAY to include author information, even for blind
% submissions: the style file will automatically remove it for you
% unless you've provided the [accepted] option to the icml2024
% package.

% List of affiliations: The first argument should be a (short)
% identifier you will use later to specify author affiliations
% Academic affiliations should list Department, University, City, Region, Country
% Industry affiliations should list Company, City, Region, Country

% You can specify symbols, otherwise they are numbered in order.
% Ideally, you should not use this facility. Affiliations will be numbered
% in order of appearance and this is the preferred way.
\icmlsetsymbol{equal}{*}

\begin{icmlauthorlist}
\icmlauthor{Yuzi Yan}{to,goo}
\icmlauthor{Jialian Li}{to}
\icmlauthor{Yipin Zhang}{to}
\icmlauthor{Dong Yan}{to}
\end{icmlauthorlist}

\icmlaffiliation{to}{Baichuan AI}
\icmlaffiliation{goo}{Tsinghua University}

\icmlcorrespondingauthor{Dong Yan}{sproblvem@gmail.com}

% You may provide any keywords that you
% find helpful for describing your paper; these are used to populate
% the "keywords" metadata in the PDF but will not be shown in the document
\icmlkeywords{Machine Learning, ICML}

\vskip 0.3in
]

% this must go after the closing bracket ] following \twocolumn[ ...

% This command actually creates the footnote in the first column
% listing the affiliations and the copyright notice.
% The command takes one argument, which is text to display at the start of the footnote.
% The \icmlEqualContribution command is standard text for equal contribution.
% Remove it (just {}) if you do not need this facility.

\printAffiliationsAndNotice{}  % leave blank if no need to mention equal contribution
% \printAffiliationsAndNotice{\icmlEqualContribution} % otherwise use the standard text.

\begin{abstract}
This paper presents an in-depth examination of the evolution and interplay of cognitive and expressive capabilities in large language models (LLMs), with a specific focus on Baichuan-7B and Baichuan-33B, an advanced bilingual (Chinese and English) LLM series. We define and explore the model's cognitive and expressive capabilities through linear representations across three critical phases: Pretraining, Supervised Fine-Tuning (SFT), and Reinforcement Learning from Human Feedback (RLHF). Cognitive capability is defined as the quantity and quality of information conveyed by the neuron output vectors within the network, similar to the neural signal processing in human cognition. Expressive capability is defined as the model’s capability to produce word-level output. Our findings unveil a sequential development pattern, where cognitive abilities are largely established during Pretraining, whereas expressive abilities predominantly advance during SFT and RLHF. Statistical analyses confirm a significant correlation between the two capabilities, suggesting that cognitive capacity may limit expressive potential. The paper also explores the theoretical underpinnings of these divergent developmental trajectories and their connection to the LLMs' architectural design. Moreover, we evaluate various optimization-independent strategies, such as few-shot learning and repeated sampling, which bridge the gap between cognitive and expressive capabilities. This research reveals the potential connection between the hidden space and the output space, contributing valuable insights into the interpretability and controllability of their training processes.
\end{abstract}

\section{Introduction}
\label{sec:introduction}
Large Language Models (LLMs) are profoundly transforming the way we work and live. To train these models, computational power worth billions is used daily in the three-phase paradigm of Pretraining, Supervised Fine-Tuning (SFT), and Reinforcement Learning from Human Feedback (RLHF). However, the specific roles of these three stages are only broadly understood: Pretraining primarily encodes knowledge, SFT aligns question-answer formats, and RLHF refines outputs via human feedback~\cite{achiam2023gpt}. Evidently, this understanding is on the level of behavioral patterns and does not aid in comprehending LLMs from a capability perspective, nor does it guide us on how to refine the training process to enhance and control the model's proficiency in various tasks.

To analyze the three-phase training paradigm from a capability perspective, researchers have introduced the concept of Alignment Tax to articulate the discrepancy between the model's inherent capabilities and its outward performance~\cite{lightman2023let, ouyang2022training, askell2021general}. Beside the training paradigm, prompt engineering also significantly influences the performance exhibited by the model. These pieces of evidence point towards a hypothesis that sometimes LLMs internally comprehend and encode the answer to a question in the internal representations but struggles to output it effectively.

Prior research in interpretability has culminated in considerable breakthroughs, aiding in demystifying the internal processes of LLMs. In~\citet{zou2023representation}, the authors propose representation engineering (RepE) as an advanced method to improve AI transparency. Meanwhile, in~\citet{park2023linear}, the study suggests the possibility of a linear space structure within neuron-level representations. Additionally, probing-based explanation techniques offer novel insights into the abstract abilities of LLMs, as explored in~\citet{zhao2023explainability}.

In this paper, we define and quantify the \emph{cognitive capability} and the \emph{expressive capability} of a LLM and explore the establishment process of them. The \emph{cognitive capability} is defined by the quantity and quality of information conveyed by the neuron output vectors within the network, similar to the neural signal processing in human cognition. This definition corresponds to the way the human brain processes information and makes sense of the world. The definition of cognitive capability exploits linear representations within the hidden space, obtained particularly from a selected intermediate layer.  On the other hand, The \emph{expressive capability} is defined as the model’s capability to produce word-level output, similar to the human ability to express thoughts or feelings by language, art, or other means. Our study includes a comprehensive series of experiments and analyses carried out during the Pretraining, SFT, and RLHF phases of the Baichuan-7B and Baichuan-33B. These models are part of an advanced bilingual LLM series Baichuan2~\cite{baichuan2023baichuan2}. Notably, Baichuan-7B is an open-source model, whereas Baichuan-33B is a closed-source model. We present the following key findings: 1) Cognitive and expressive capabilities evolve at different paces. Specifically, cognitive capability is primarily established during the Pretraining stage, whereas expressive capability is developed during the SFT and RLHF stages, with SFT playing a more significant role. 2) A robust statistical correlation exists between cognitive and expressive capabilities. The cognitive capability sets the upper boundary for the expressive capability. 3) Our research illustrates that specific techniques, including few-shot learning, repeated sampling, and prompt engineering, can efficaciously bridge the gap between a LLM's expressive and cognitive capabilities.

In addition, we delve into the internal mechanisms governing the development of cognitive and expressive capabilities, along with a theoretical analysis of the gap between them, in Section~\ref{sec: theoretical analysis}. We conduct multiple experiments to underpin our hypotheses. Specifically, the discrepancy between these capabilities may stem from the differences in linear separability between the embedding space of neuron output and the token-level semantic space. From the standpoint of the LLM's architecture, the diminution of this gap during the SFT/RLHF stage could be attributed to enhancements in the vocabulary linear layer at this phase. We anticipate that these discoveries will offer valuable insights into the training process of LLMs.

\begin{figure}[t]
\begin{center}
\centerline{\includegraphics[width=\columnwidth]{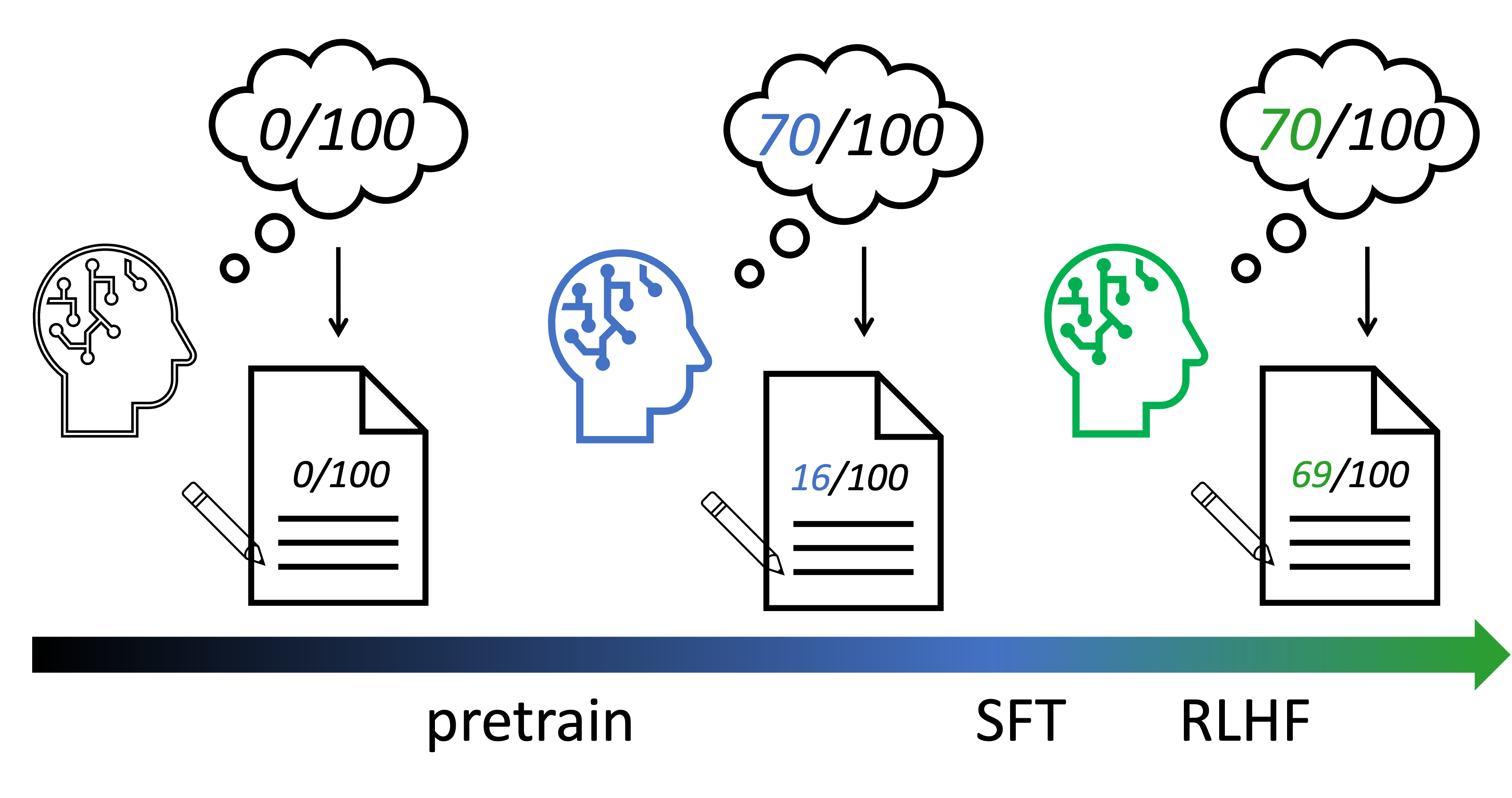}}
\caption{
Schematic representation of the asynchronous capabilities development process in LLMs. Initially, the model lacks the ability to comprehend questions or generate relevant responses. Through the Pretraining phase, the LLM primarily acquires cognitive capabilities, though its ability to articulate responses remains underdeveloped. Subsequent SFT and RLHF enhance the model's expressive capability, aligning it closely with the cognitive skills.
}
\vspace{-25pt} % 减少标题下方的空间
\label{fig: schematic_diagram}
\end{center}
\end{figure}

\section{Related Work}
\label{sec:related work}
It is very appealing to explain the model's capability by analyzing the hidden space's linear properties of the language model, and it has attracted a lot of research attention recently. These works significantly inspired this paper. A comparative analysis is then presented to highlight the relationship and the distinctions between our work and these prior studies.

\textbf{Linear subspaces and geometry in language representations.}  
The hypothesis of linear subspaces was initially observed empirically in the context of word embeddings by~\citet{mikolov2013distributed}. Similar structures have been observed in cross-lingual word embeddings~\cite{mikolov2013exploiting}, as well as in sentence embeddings and the representational spaces of Transformer-based LLMs~\cite{hernandez2023linearity}. There is a significant body of work studying the geometry of word-level or sentence-level representations~\cite{arora2016latent, mimno2017strange, li2023emergent, park2023linear}. These observations motivate our approach to measuring cognitive capability in LLMs using linear representations in the hidden space. This paper lends further credibility to the analysis of language model characteristics within a linear hidden space by demonstrating a strong correlation between the measured cognitive capabilities and the model's upper limit of expressiveness.

\textbf{Promoting LLM by linear representation.}
Recent advancements have leveraged linear representations to augment the capabilities of LLMs. In~\citet{liu2023aligning}, the authors introduce an innovative RLHF approach, which refines linear representations to align closely with high-level human preferences, thereby enabling more precise control over model behavior and enhancing its performance. The findings of this study suggest two critical insights: firstly, certain linear representations within the hidden layers may provide information that is at least as significant as that conveyed by token-level outputs. Secondly, these linear representations can serve as potent indicators for directing the model's refinement. In our work, we substantiate the first hypothesis by delineating the gap between cognitive capabilities, as defined by linear representations, and expressive capabilities, as indicated by direct token-level outputs. Furthermore, we lend partial support to the second hypothesis by identifying strategies to bridge this gap without parameter optimization.

\textbf{Measurement and mechanistic interpretability.} 
A considerable volume of research has been dedicated to the exploration of linear representations for both interpreting (probing)~\cite{alain2016understanding,kim2018interpretability} and manipulating (steering)~\cite{turner2023activation} the behavior of models. Notably, the work presented in~\citet{zou2023representation} posits that engaging with concept-level representations within LLMs can substantially enhance the model's proficiency in specific concepts such as truthfulness and honesty. This discovery underscores the potential of linear representations significantly augment or modulate the expressive prowess of models from a top-down perspective. Complementing this viewpoint, our study explore the interplay between linear representations and the structural design of LLMs, offering a bottom-up analysis.

\section{Main Results}
\label{sec: main results}
This section defines \textit{cognitive} and \textit{expressive} capabilities in LLMs and outlines their quantification methods. We then present experimental results from public datasets, highlighting the asynchronous development of these capabilities during training. Finally, we demonstrate their statistical correlation, underscoring their interdependence in LLM performance.

\subsection{Definitions and Quantification of Cognitive and Expressive Capabilities}
\label{subsec: The definitions and the quantification of cognitive and expressive capabilities}
We conceptualize the inference process in a LLM as follows: Given an input prompt $x$ comprising $n$ tokens, where $x \in \cT^n$ and $\cT$ denotes the token-level space, the LLM initially maps $x$ to a high-dimensional vector $c \in \cR^m$ through a mapping function $f(\cdot)$. The architecture of the LLM, such as in prevalent decoder-only models like Llama 2~\cite{touvron2023llama} or GPT~\cite{achiam2023gpt}, determines the specifics of $f(\cdot)$, including the hidden size of the Transformer block~\cite{vaswani2017attention}, network weights, the layer from which $c$ is extracted, and other hyperparameters. Subsequently, the function $g(\cdot)$ maps $f(x)$ to the next token output $y \in \mathcal{T}$, influenced by the model's remaining architecture, notably including a critical vocabulary linear layer discussed further in Section~\ref{sec: theoretical analysis}. The inference process is succinctly represented as:
\$
x\in \cT^n \xrightarrow{f(\cdot)} c \in \cR^m \xrightarrow{g(\cdot)} y \in \cT
\$

Our hypothesis, supported by empirical evidence presented in later sections, posits that the intermediate vector $c$ harbors more insightful information compared to the direct token output $y$. For instance, in binary classification tasks such as "True or False" questions, leveraging unsupervised algorithms like PCA on $c$ has shown to surpass strategies that directly analyze $y$. This suggests that LLMs may grasp the underlying problem and possess the correct solution, yet lack the capability to articulate it accurately.

We analogize $c$, emerging from the neuron outputs within the network and resembling neural signals in human cognition, as the model's cognitive capability. Conversely, the ability to produce token-level outputs is defined as the LLM's expressive capability.

To assess these capabilities, we devise experiments with single-choice questions, comprising one question ($<\text{QUESTION}>$) and multiple choices ($<\text{CHOICE}_i>$), where only one is correct. The evaluation methods are detailed below:

\textbf{Quantification of cognitive capability.}

For each single-choice question within a dataset, we pair each $(<\text{QUESTION}>, <\text{CHOICE}_i>)$ with Template A (see Appendix~\ref{appen:sec:template}) to form an input $x$. The LLM's cognitive capability on this dataset is evaluated using Algorithm~\ref{alg: representation engineering}, drawing on the unsupervised Representation Engineering (RepE) approach, which uses Principal Component Analysis (PCA)~\cite{zou2023representation}.

\textbf{Quantification of expressive capability}: 
In the test set $D_{\text{test}}$, we wrap each $(<\text{QUESTION}>, \{<\text{CHOICE}_i\}_{i=1}^{4})$ with Template B (see Appendix~\ref{appen:sec:template}). For a given model, the wrapped prompt serve as inputs to calculate the accuracy of direct token responses, defining the expressive capability's quantification. It is crucial to note that our method diverges from frameworks like~\citet{eval-harness}, which employ greedy search to assess the likelihood of each option in its entirety, as opposed to our focus on the model's direct token outputs.

\begin{algorithm}[t]
\caption{Cognitive capability quantification}
\begin{algorithmic}[1]
    \REQUIRE Model $M$, training set $D_{\text{train}}$, test set $D_{\text{test}}$.
    \STATE Format the input prompts $\{x^{\text{train}}\}$ from $D_{\text{train}}$ and $\{x^{\text{test}}\}$ from $D_{\text{test}}$ using Template A.
    \STATE \textbf{PCA Direction Extraction:}
    \FOR{each Transformer block $i$}
        \STATE Use $\{x^{\text{train}}\}$ as input and extract embeddings $\{c^{\text{train}}_i\}$ for the last token.
        \STATE Calculate a PCA direction $v_i \in \cR^m$ using $\{ c^{\text{train}}_i \}$.
        \STATE Determine the sign function: $S_i \in \{\argmin, \argmax\}$ based on the principle component extracted and the correct answers.
    \ENDFOR
    \STATE \textbf{Evaluation:}
    \FOR{each Transformer block $i$}
        \STATE Use $\{x^{\text{test}}\}$ as input and extract embeddings $\{c^{\text{test}}_i\}$ for the last token.
        \STATE Project $\{c^{\text{test}}_i\}$ onto $v_i$ and choose one answer by $S_i$: $S_i(v_i^\Tb \cdot c^{\text{test}}_i )$.
        \STATE Compare the chosen answers and the correct answers, calculate the accuracy: $\text{Acc}_i$
    \ENDFOR
    \STATE Compute maximum evaluation accuracy across all Transformer blocks as the quantification of the cognitive capability: $\max_i(\{ \text{Acc}_i \})$.
\end{algorithmic}
\label{alg: representation engineering}
\end{algorithm}

\subsection{Datasets and Experimental Setup}
We carry out our quantification experiments using four standard benchmark datasets: OpenbookQA \cite{mihaylov2018can}, CommonSenseQA \cite{talmor2018commonsenseqa}, RACE \cite{lai2017race} and ARC \cite{clark2018think}. OpenBookQA is designed to test a language model's ability for text understanding and reasoning. It focuses on the application of common sense and general knowledge in answering questions. CommonSenseQA is a benchmark for testing the common sense of AI systems. It includes questions that require an understanding of everyday concepts and relationships between objects and ideas. The RACE dataset is a large-scale reading comprehension dataset collected from English exams for middle and high school Chinese students. It consists of passages and corresponding single-choice questions. The AI2 Reasoning Challenge (ARC) aims to evaluate a system's reasoning ability and understanding of scientific texts, offering two levels of difficulty, referred to as ARC-challenge and ARC-easy. All the datasets above can be formatted as single-choice questions with 4 options.

Our experiments use checkpoints from the Pretraining, SFT, and RLHF stages of the in-house developed Baichuan-7B and Baichuan-33B, both decoder-only, bilingual LLMs. The 7B model is openly available~\cite{baichuan2023baichuan2}, and the 33B model extends the 7B architecture with increased parameters. For RLHF, we implement the Proximal Policy Optimization (PPO) strategy, as elaborated in \citet{achiam2023gpt}.

\subsection{Pretraining: Building Cognitive Capability}
\label{subsec: Pretraining building the cognitive capability}

\begin{figure}[t]
\begin{center}
\centerline{\includegraphics[width=\columnwidth]{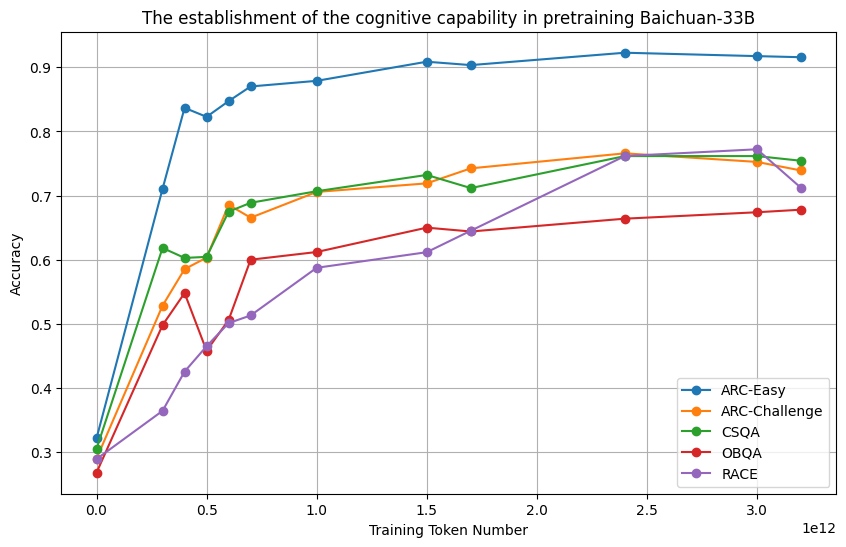}}
\caption{
Progression of cognitive capability during the Pretraining stage in Baichuan-33B, as quantified by linear representations. The graph illustrates a stabilization in cognitive performance when the volume of training data reaches approximately 2.4T. 
}
\label{fig: pretrain pca}
\vspace{-30pt} % 减少标题下方的空间
\end{center}
\end{figure}
This section details the development of cognitive capability during the Pretraining phase for Baichuan-33B and Baichuan-7B. Both models were trained from the ground up, with training data incrementally increased up to 3.2T. The progression of cognitive capability in Baichuan-33B, assessed by Algorithm~\ref{alg: representation engineering}, is depicted in Figure~\ref{fig: pretrain pca}. For Baichuan-7B, corresponding findings are presented in Figure~\ref{appen: fig: pretrain pca in 7B} within Appendix~\ref{appen:subsec:Pretraining process of Baichuan-7B}.

In both Baichuan-33B and Baichuan-7B, we note a swift initial improvement in cognitive capability that tapers off with increased training data. Initially, both models exhibit decision-making akin to random guessing, with accuracy around 0.25, indicative of their nascent state. The cognitive capability's growth stabilizes around 2.4T of data for Baichuan-33B and 1.5T for Baichuan-7B, suggesting a nearing to the models' cognitive limits in tasks like reasoning, commonsense understanding, and information retrieval across various datasets. Notably, the model with fewer parameters, Baichuan-7B, demonstrates lower accuracy, reflecting its lower cognitive capacity compared to Baichuan-33B.

Additionally, we applied the general lm-evaluation-harness framework~\cite{eval-harness} to both models, which employs greedy search to evaluate each option's probability for answer selection. The outcomes are depicted in Figure~\ref{fig:lm evaluation harness in Baichuan-33B} for Baichuan-33B and in Figure~\ref{fig:lm evaluation harness in Baichuan-7B} for Baichuan-7B in Appendix~\ref{appen:subsec:The gap between the quantified cognitive capability and the lm-evaluation-harness performance}. Despite the similarity in pattern between the two curves, a noticeable discrepancy exists. Both sets of results distinctly illustrate the progression of the models' intrinsic cognitive capabilities, yet highlight a shortfall in expressive capability, which we will explore further in the following section.

\subsection{SFT and RLHF: Aligning Expressive and Cognitive Capabilities}
This section explores how the SFT and RLHF stages align cognitive and expressive capabilities. Post-Pretraining, despite high cognitive accuracy, the models often fail to deliver correct token-level answers. We assess expressive capability as outlined in Section~\ref{subsec: The definitions and the quantification of cognitive and expressive capabilities}, finding both zero-shot and few-shot performances significantly lagging behind cognitive accuracy measured by Algorithm~\ref{alg: representation engineering} (see Figures~\ref{fig:gap between cognitive capability and expressive capability in SFT and PPO} and~\ref{fig:bridging the gap between the expressive capability and the cognitive capability}). Case studies in Appendix~\ref{appen:sec:case study} reveal instances of incoherent responses, suggesting that while advanced cognition may be achieved late in Pretraining, guiding accurate expressions through zero-shot or few-shot approaches remains challenging. SFT and RLHF effectively reduce this cognitive-expressive discrepancy.

%Based on the pretrained Baichuan-33B model with 3.0T data, we trained a series of SFT-version and RLHF-version models. In the SFT process, we train 4 epochs, using 1M tokens for each epoch. In the RLHF process, we first train a reward model based on the pretrained Baichuan-33B model with some preference-ranked data annotated by our in-house annotation team, then conduct the standard PPO pipeline starting from the obtained SFT model at epoch 4. 
Leveraging the pretrained Baichuan-33B model with 3.0T training data, we developed variants through SFT and RLHF. The SFT phase involved training over 4 epochs, each with 1M tokens. For RLHF, we initially trained a reward model on preference-ranked data annotated by our in-house annotation team, using the pretrained Baichuan-33B as a base. This was followed by implementing the standard PPO pipeline, commencing from the SFT model at epoch 4.

%We compare the cognitive capability and the expressive capability on a series of SFT models with different epoch and the final PPO model. The results are shown in Figure~\ref{fig:gap between cognitive capability and expressive capability in SFT and PPO}. There are two observations: 1. The cognitive capability will not largely changed in the SFT or RLHF process. 2. The expressive capability will increase rapidly during the SFT process, and will eventually approach the cognitive capability, but will not exceed. The observation 1 supports the point of section~\ref{subsec: Pretraining building the cognitive capability} that the establishment of the cognition mainly occurs in the Pretraining stage. The observation 2 on the one hand illustrates the role of SFT and RLHF (mainly SFT) in the model training process, and on the other hand implies that cognitive capability may represent the potential upper limit of the model's expressive capability. In Section~\ref{sec:Methods for Bridging the Gap}, we will discuss methods to stimulate the expressive capability to narrow down the gap. 
We assessed cognitive and expressive capabilities across various SFT epochs and the concluding PPO model, with findings illustrated in Figure~\ref{fig:gap between cognitive capability and expressive capability in SFT and PPO}. Key observations include: 1) Cognitive capability remains relatively stable throughout SFT and RLHF phases. 2) Expressive capability significantly improves during SFT, eventually nearing but not surpassing cognitive capability. Observation 1 corroborates Section~\ref{subsec: Pretraining building the cognitive capability}'s assertion that cognitive development primarily transpires during Pretraining. Observation 2 highlights the pivotal role of SFT (and to a lesser extent, RLHF) in enhancing expressive capability, suggesting cognitive capability as a potential ceiling for expressiveness. Approaches to optimize expressive capability and narrow this gap are explored in Section~\ref{sec:Methods for Bridging the Gap}.

\begin{remark}
The quantification of cognitive capability, conducted in the hidden space via Principal Component Analysis (PCA), presents a non-trivial approach for assessing the internal capabilities of LLMs. This opens avenues for employing additional linear analysis techniques for model analysis or control, as evidenced in recent studies such as~\citet{zou2023representation, liu2023aligning}. These methods will constitute the core of our forthcoming research endeavors.
\end{remark}
\begin{figure}[t]
\begin{center}
\centerline{\includegraphics[width=\columnwidth]{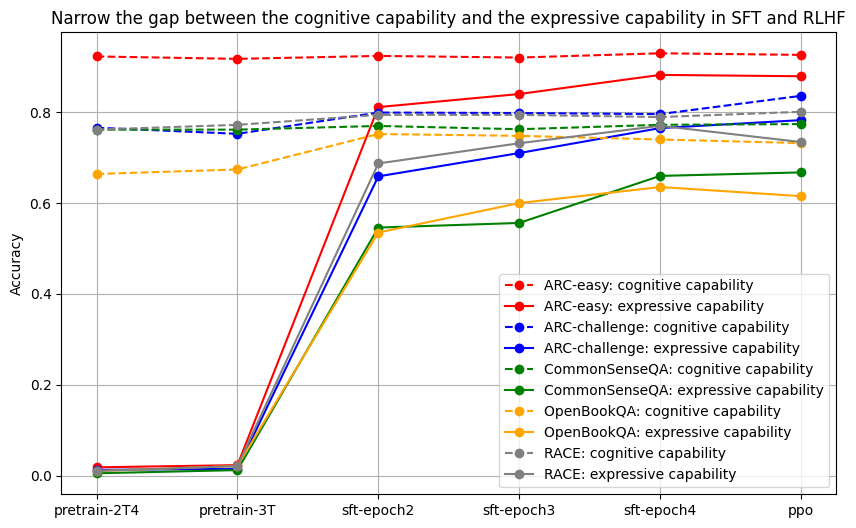}}
\caption{
%The figure shows the gap narrowing process of the expressive capability and the cognitive capability in the SFT and RLHF stages. Each SFT epoch processes 1 million tokens. The dotted curve represents the cognitive capability and the solid curve represents the expressive capability. The cognitive capability established by the model in the Pretraining stage sets an upper limit for the expressive capability, and the gap is narrowed during the process of SFT and RLHF.
Illustration of the diminishing gap between expressive and cognitive capabilities in SFT and RLHF. Each SFT epoch processes 1 million tokens. The dotted line signifies the cognitive capability, established during the Pretraining phase and acting as the upper boundary for expressive capability. The solid line represents the expressive capability. The diagram highlights the gradual reduction of the disparity between these capabilities as the model undergoes further refinement through SFT and RLHF.
}
\vspace{-25pt} % 减少标题下方的空间
\label{fig:gap between cognitive capability and expressive capability in SFT and PPO}
\end{center}
\end{figure}

% \begin{figure}[t]
% \begin{center}
% \centerline{\includegraphics[width=\columnwidth]{figure/gap_race.png}}
% \caption{An example in RACE: The complete establishment process of the cognitive capability and the expressive capability in Pretraining, SFT and RLHF. }
% \label{fig:gap in race}
% \end{center}
% \end{figure}

\subsection{Statistical Correlation between Cognitive Capability and Expressive Capability}
%In this subsection, we calculate the consistency of two quantification methods and examine the correlation between cognitive capability and expressive capability through hypothesis testing. The results are detailly shown in the Table~\ref{appen: table: hypothesis testing probability} in Appendix~\ref{appen:subsec:hypothesis testing} . We give an example for RACE in Table~\ref{table: hypothesis testing probability}.
In this subsection, we evaluate the consistency between two quantification methods and investigate the correlation between cognitive and expressive capabilities using hypothesis testing. Detailed results are presented in Table~\ref{appen: table: hypothesis testing probability} within Appendix~\ref{appen:subsec:hypothesis testing}. An illustrative example for the RACE dataset is provided in Table~\ref{table: hypothesis testing probability}.

%In the process of quantifying both capabilities, the LLM actually engages in answering single-choice questions. When both quantification approaches yield identical results for a given question—whether both are correct or both incorrect—we deem the methods to be consistent for that question. We determine the consistency for a particular model on a dataset by calculating the ratio of questions for which the two methods are consistent to the total number of questions. The consistency results are documented in Column 3 of Table~\ref{table: hypothesis testing probability}.
During quantification of the capabilities, the LLM answers single-choice questions. Consistency between the two quantification methods for a question is established when both yield the same outcome—correct or incorrect. We assess a model's consistency on a dataset by computing the ratio of questions where the methods concur to the total question count. These consistency metrics are documented in Column 3 of Table~\ref{table: hypothesis testing probability}. An upward trend in consistency is noted with progressing SFT epochs, indicative of the expressive capability increasingly mirroring the stable cognitive capability at a granular level.

% To verify the correlation between the cognition capability and the expressive capability, we conduct the hypothesis testing. The accuracy achieved through the quantification of cognitive capability is designated as cognitive accuracy $a^{\text{cog}}$,  and the accuracy derived from quantifying expressive capability is termed expressive accuracy $a^{\text{exp}}$. Our initial hypothesis posits that the expressive accuracy $a^{\text{exp}}$ and the cognitive accuracy $a^{\text{cog}}$ are independent random variables. Under this hypothesis, the number of the questions that two methods are consistent is expected to follow a binomial distribution $\cB(s, 1-(1-a^{\text{exp}}) \cdot (1-a^{\text{cog}}))$ where $s$ represents the total number of questions.
% Subsequently, we compute the probability of achieving the observed level of consistency. In repeated experiments across various datasets, this probability consistently falls below $0.01\%$, indicating the initial hypothesis is unlikely, thereby implying a significant correlation between these two variables.
To assess the correlation between cognitive and expressive capabilities, we employ hypothesis testing. We denote the accuracy from cognitive capability quantification as cognitive accuracy $a^{\text{cog}}$, and that from expressive capability quantification as expressive accuracy $a^{\text{exp}}$. Our null hypothesis assumes $a^{\text{exp}}$ and $a^{\text{cog}}$ are independent. Under this premise, the consistency count across methods for a set of questions is expected to adhere to a binomial distribution $\mathcal{B}(s, 1-(1-a^{\text{exp}}) \times (1-a^{\text{cog}}))$, where $s$ is the total question count. We then calculate the likelihood of observing the actual consistency level, finding it to consistently be less than $0.01\%$ across diverse datasets. Such low probabilities strongly suggest the null hypothesis to be improbable, thereby indicating a significant correlation between $a^{\text{exp}}$ and $a^{\text{cog}}$.

\begin{table*}[t]
\caption{The statistical correlation analysis of the cognitive capability and the expressive capability in RACE. H-Test probability stands for hypothesis testing probability with the hypothesis that these two capabilities are irrelevant. We refer the expressive accuracy and the cognitive accuracy to the accuracy that are obtained by the quantification of the expressive capability and the cognitive capability as mentioned in Section~\ref{sec: main results}. }
\label{table: hypothesis testing probability}
\vskip 0.15in
\begin{center}
\begin{small}
\begin{sc}
\begin{tabular}{p{2.5cm}lcM{1.8cm}M{1.8cm}M{2.1cm}M{2cm}}
\toprule
Data set & Model & Consistency & Expressive Accuracy & Cognitive Accuracy & Binomial Distribution & H-Test probability\\
\midrule
RACE             & SFT Epoch2& 72.26\% & 68.72\% & 79.42\% & $\cB(3451, 0.6101)$& $<$0.01\%\\
                 & SFT Epoch3& 73.24\% & 73.17\% & 78.94\% & $\cB(3451, 0.6341)$& $<$0.01\%\\
                 & SFT Epoch4& 74.24\% & 78.94\% & 74.00\% & $\cB(3451, 0.6563)$& $<$0.01\%\\
                 & PPO       & 73.92\% & 77.01\% & 80.10\% & $\cB(3451, 0.6410)$& $<$0.01\%\\
\bottomrule
\end{tabular}
\end{sc}
\end{small}
\end{center}
\vskip -0.1in
\end{table*}

\subsection{Assessment of Cognitive Convergence across Training Phases}
% In this subsection, we illustrate the convergence process of the cognitive capabilities by examining the consistency between consecutive checkpoints during the Pretraining, SFT, and RLHF phases. Specifically, for two consecutive model checkpoints, we quantifying the cognitive capability respectively and calculate the ratio of questions for which the two models are NOT consistent to the total number of questions. The lower the ratio, the smaller the differences between the models from the cognitive perspective. The results are shown in Figure.~\ref{fig:inconsistency trend}.
In this subsection, we explore the convergence of the cognitive capability by assessing consistency across consecutive checkpoints. We quantify cognitive capabilities for adjacent model checkpoints and compute the inconsistency ratio—the proportion of questions where the two models diverge—to the total question count. A lower inconsistency ratio indicates smaller cognitive discrepancies between the models. The findings are depicted in Figure~\ref{fig:inconsistency trend}.

% Our analysis reveals a progressive decline in inconsistency between two consecutive models throughout the Pretraining phase, indicating increasing stability in cognitive responses. However, upon transitioning to SFT with a new corpus, we initially observe a spike in inconsistency, which then resumes its decline. Interestingly, inconsistency slightly increases again after RLHF, suggesting that the onset of SFT and RLHF introduces a temporary uncertainty in cognitive processing, despite the overall trend towards greater consistency during Pretraining.
The results indicate a steady decrease in inconsistency among consecutive models during Pretraining, signaling enhanced stability in cognitive responses. Transitioning to SFT with a novel corpus leads to an initial increase in inconsistency, which subsequently diminishes. Notably, inconsistency experiences a minor uptick following RLHF, hinting that the introduction of SFT and RLHF momentarily injects cognitive uncertainty, despite a prevailing trend of improved consistency throughout Pretraining.

\begin{figure}[t]
\begin{center}
\centerline{\includegraphics[width=\columnwidth]{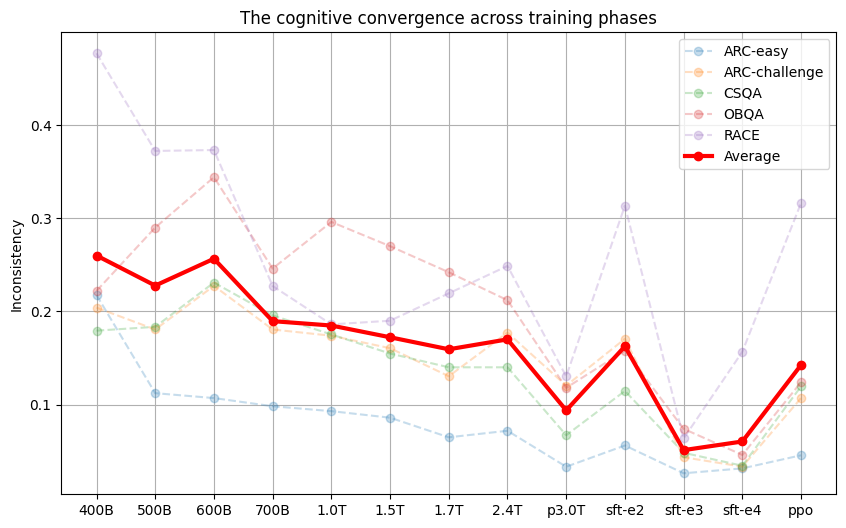}}
\caption{Convergence of cognitive capabilities by assessing consistency across consecutive checkpoints. The y-axis quantifies the discrepancy in judgments between each model and its predecessor. The red solid line is the average result. }
\vspace{-10pt} % 减少标题下方的空间
\label{fig:inconsistency trend}
\end{center}
\end{figure}

\section{Theoretical Analysis}
\label{sec: theoretical analysis}

\subsection{Explanation of the Capability Gap}
\label{subsec: explanation of the capability gap}
The following theorem articulates our rationale for the observed gap between cognitive and expressive capabilities.
\begin{theorem}
The gap between cognitive and expressive capabilities stems from the superior mapping efficiency of the function \( f(\cdot) \) compared to \( g(\cdot) \), along with the greater linear separability afforded by the hidden space \( \mathcal{R}^m \) over the token-level space \( \mathcal{T} \).
\end{theorem}

% To validate this theorem, we employ a basic linear classifier within $\cR^m$ and $\cT$, leveraging the HalluQA~\cite{cheng2023evaluating} as our dataset, a dataset for evaluating hallucinations in Chinese large language models through counterfactual question-answer pairs. This dataset enables the construction of distinct positive and negative examples, crucial for our comparative experiments:
To corroborate this theorem, we use a simple linear classifier within both $\mathcal{R}^m$ and $\mathcal{T}$, utilizing the HalluQA dataset~\cite{cheng2023evaluating} designed for assessing hallucination phenomena in Chinese LLMs via counterfactual question-answer pairs. This dataset facilitates the creation of clear positive and negative examples, essential for our comparative analysis:

% \textbf{Linear SVM on $\cR^m$}: On the trainset, we wrap each (\text{QUESTION}, \{\text{ANSWER}\}) pair with Template A and extract embedding $c$ from the selected layer. We train a linear-kernel SVM on $\{c \}$ and evaluate the classification accuracy on the testset. 
\textbf{Linear SVM on $\mathcal{R}^m$}: In the trainset, each (QUESTION, \{ANSWER\}) pair is processed with Template A (see Appendix~\ref{appen:sec:template}) to extract the embedding $c$ from a chosen layer. A linear-kernel SVM is then trained on these embeddings $\{c\}$, with its classification accuracy assessed on the testset.

\textbf{Direct token generation on $\mathcal{T}$}: The model's accuracy in generating responses is evaluated on the test set, comparing against the provided correct answers as the reference.

% We mix the trainset into the SFT training data. The results are shown in Table~\ref{table: hypothesis testing probability}. We observe significant gap in the accuracy of SVM classifications versus direct token generation. Notably, while SVM classification accuracy remains relatively stable, the accuracy of direct token generation exhibits gradual improvement with increasing SFT epochs and the conduct of RLHF.
The trainset is integrated into the SFT training data. The outcomes, detailed in Table~\ref{table: SVM and direct token generation}, reveal a pronounced disparity between the accuracies of SVM classifications and direct token generation. Remarkably, the accuracy of SVM classifications stays relatively constant, whereas direct token generation accuracy progressively enhances with additional SFT epochs and through the implementation of RLHF.

This discrepancy underscores the distinct classification landscapes offered by $\cR^m$ and $\cT$. The linear SVM delineates a hyperplane in $\cR^m$ that effectively segregates the data into two categories, in contrast to the hyperplane in $\cT$ inferred by direct token generation. The comparative analysis reveals that $\cR^m$ facilitates lower intra-class variance and higher inter-class variance, indicating more pronounced class separability than $\cT$.

\begin{figure}[t]
\begin{center}
\centerline{\includegraphics[width=0.8\columnwidth]{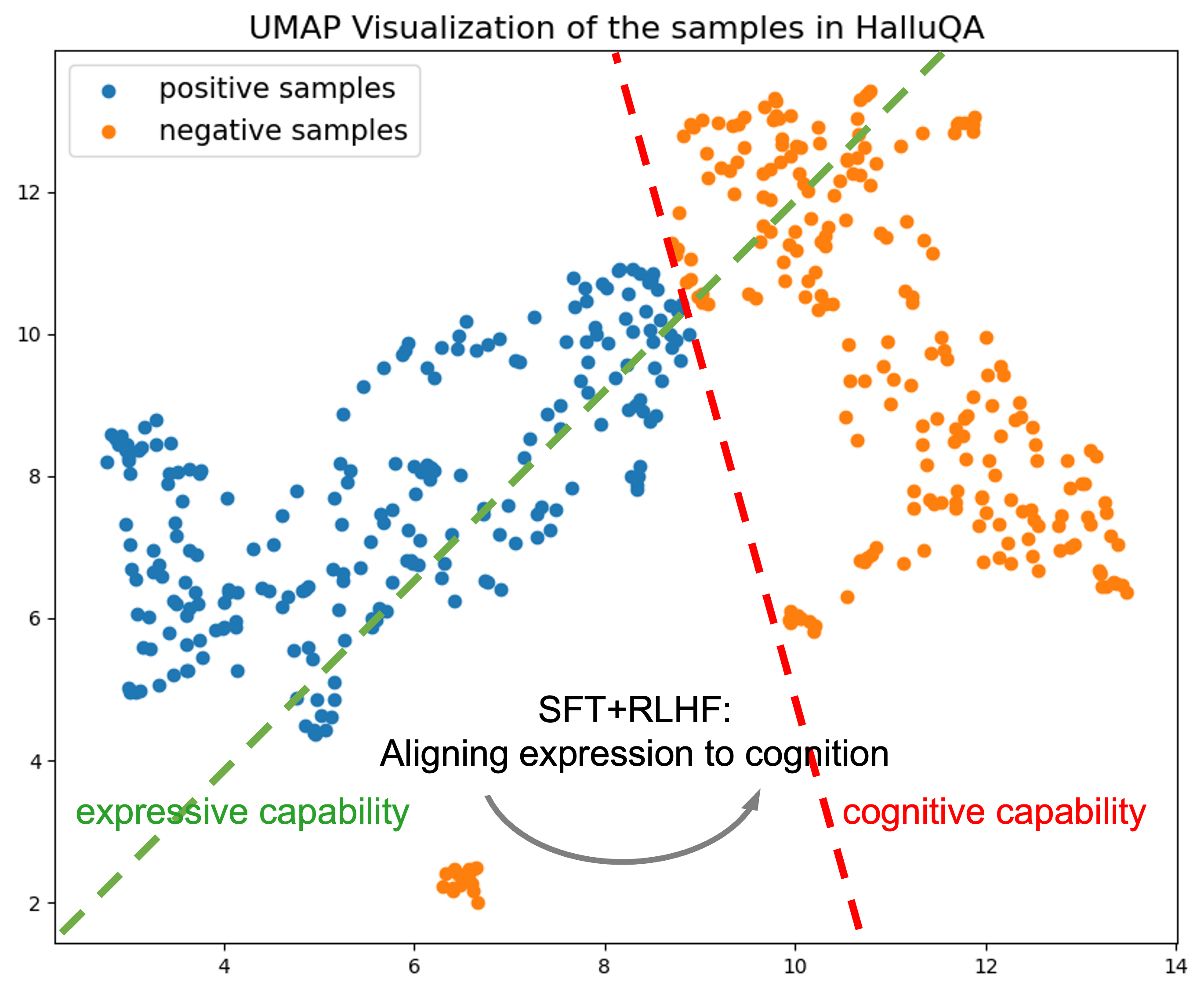}}
\caption{UMAP visualization of HalluQA classifier in Baichuan-33B. The red line represents the delineation by SVM on neuron output while the green line represents that of token-level output. The blue dots and orange dots represents positive samples and negative samples in the datasets respectively. SFT and RLHF demonstrates the potential to align the expressive capability to the cognition capability in the fine-tuning stages.}
\label{appen:fig:umap visualization}
\vspace{-25pt}
\end{center}
\end{figure}

Figure~\ref{appen:fig:umap visualization} conceptualizes this distinction through UMAP \cite{mcinnes2018umap} reduction $\cT$ from the final transformer block to two dimensions. The delineation by the SVM and direct token generation classifiers, represented by red and green lines respectively, visually captures the gap between cognitive and expressive capabilities. SFT and RLHF demonstrates the potential to bridge this gap.

\begin{table}[t]
\caption{The performance gap between the Linear SVM and the direct token generation on HalluQA in Baichuan-33B. }
\label{table: SVM and direct token generation}
% \vskip 0.15in
\begin{center}
\begin{small}
\begin{sc}
\begin{tabular}{llllll}
\toprule
Model & Pretrain & SFT-2 & SFT-3 & SFT-4 & PPO \\
\midrule
L-SVM & 0.868 & 0.875 & 0.868 & 0.868 & 0.868 \\
Direct & 0.0607 & 0.493 & 0.509 & 0.513 & 0.563 \\
\bottomrule
\end{tabular}
\end{sc}
\end{small}
\end{center}
\vskip -0.1in
\end{table}

\subsection{Establishment of Cognitive Capability}
\label{subsec: the establishment of the cognitive capability}
% To demonstrate the development of cognitive capabilities and identify the specific layer that embodies the cognition capability within the LLM, we apply Algorithm~\ref{alg: representation engineering}, selecting different layers for analysis in both Baichuan-7B and Baichuan-33B. The Baichuan-33B results are shown in Figure~\ref{fig:the establishment of cognitive capability}, and the Baichuan-7B results are in Figure~\ref{appen:fig:the establishment of cognitive capability for Baichuan-7B} in Appendix~\ref{appen:subsec:Pretraining process of Baichuan-7B}.
To elucidate the evolution of cognitive capabilities and pinpoint the layer that most significantly represents cognitive capability within the LLM, we apply Algorithm~\ref{alg: representation engineering} to examine various layers in both Baichuan-7B and Baichuan-33B models. The findings for Baichuan-33B are depicted in Figure~\ref{fig:the establishment of cognitive capability}, while results for Baichuan-7B are presented in Figure~\ref{appen:fig:the establishment of cognitive capability for Baichuan-7B} within Appendix~\ref{appen:subsec:Pretraining process of Baichuan-7B}.

% The peak accuracy of the curve serves as an indicator of cognitive capability strength. In the Pretraining and SFT phase, the LLM's cognitive capability shows a consistent upward trend with increasing training data. The impact of RLHF training on cognitive capability is ambiguous and likely depends on the data distribution used in training the reward model. In evaluations using ARC-easy, ARC-challenge, and CommonSenseQA datasets, the PPO-enhanced model surpasses the SFT counterpart. Conversely, in OpenBookQA assessments, PPO implementation marginally reduces performance, possibly due to the reward model's bias towards specific types of data.
The peak accuracy of the curve serves as an indicator of cognitive capability. During the Pretraining and SFT phases, there is a notable increase in cognitive capability as training data volume expands. The effect of RLHF on cognitive capability is ambiguous, seemingly influenced by the training data's distribution for the reward model. Performance assessments on datasets like ARC-easy, ARC-challenge, and CommonSenseQA show the PPO-enhanced model outperforming its SFT-only version. However, in OpenBookQA evaluations, the introduction of PPO slightly detracts from performance, hinting at a potential data-specific bias in the reward model.

The curve's trajectory offers valuable insights into the model's capability development. In the initial phase of Pretraining, with training data under 1T, the curve's peak typically aligns with the model's mid-section. For instance, in the Baichuan-33B model, using 700B tokens for Pretraining, the peak value, as assessed through ARC-Challenge, is identified in the 26th Transformer block, depicted by the orange curve in the third figure of Figure~\ref{fig:the establishment of cognitive capability}. As the training data expands to 1.5T and 3T, the accuracy within the model's final 30 layers stabilizes at a high level. This pattern, observed across various tests, suggests a characteristic feature of the Pretraining phase nearing convergence.

\begin{remark}
The establishment process of layer-wise cognitive capability is divided into two periods, the bell curve (in the early stages of Pretraining, with less training data) and the plateau curve (in the late stages of Pretraining, with more training data). In the plateau period, the cognitive capability reaches its peak at a certain intermediate layer. Based on this, we believe that cognitive capability may be established in the first few layers of the model, while the pleatue layers continuously strengthen this cognitive capability, and are mapped to expressive capability in the final linear layer. This phenomenon may result largely from \emph{Residual Connection} and \emph{pre-Layer Normalization} in the model architecture. Besides, the appearance of the cognitive capability plateau curve may represent some redundancy of the model. The results have been observed and discussed in a previous work~\cite{men2024shortgpt} and we emphasize it here. We leave more discussion in Section~\ref{supple:sec:Further Discussion about the Cognition Establishment Process} in the supplementary material.
\end{remark}

The SFT and RLHF phase mainly influences the performance of the model's final layers. The results in ARC-challenge and OpenBookQA (see Figure~\ref{appen:fig:the establishment of cognitive capability}) illustrates a notable performance dip in the ultimate layer of the pretrain-3T model compared to its preceding layers. Nonetheless, SFT and RLHF helps in rectifying this performance gap. The efficacy of the model's final layers, especially the last one, is significantly associated with its expressive capability. This relationship will be explored in the following subsection.

\begin{figure}[t]
\begin{center}
\centerline{\includegraphics[width=\columnwidth]{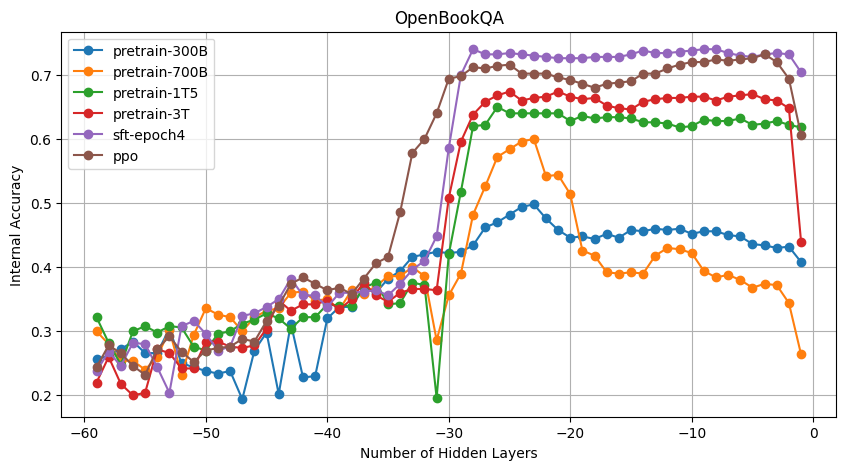}}
\caption{Layer-wise performance of linear representations in Baichuan-33, shedding light on the intricate architecture underlying cognitive capability formation.}
\vspace{-15pt} % 减少标题下方的空间
\label{fig:the establishment of cognitive capability}
\end{center}
\end{figure}

\subsection{Establishment of Expressive Capability}
This subsection explores the significance of the vocabulary linear layer in shaping the expressive capabilities of a model. Positioned as the final MLP (Multi-Layer Perceptron) layer with trainable parameters within a decoder-only LLM, the vocabulary linear layer maps the last transformer block's output to the vocabulary space $\cT$. During the model's greedy output generation, this layer effectively assesses the similarity between its input and each of its row vectors (each corresponding to a token in the vocabulary), ultimately selecting the token that exhibits the highest similarity for its prediction. This process not only underpins the model's ability to generate coherent and contextually relevant text but also forges a structural link between cognitive and expressive capabilities. Earlier sections have illustrated that as the model approaches the convergence point during Pretraining, the output performance of the last few transformer blocks stabilizes, reflecting the model's cognitive capabilities. Consequently, the effective training of the vocabulary linear layer is paramount, as it directly influences the model's ability to articulate its 'thoughts' and knowledge accurately.

To demonstrate the dynamics of the vocabulary linear layer, we designed an experiment using a set of prompts $\{ x_i \}_{i=1}^m$ and performed inference across a series of models. We define the output from the final transformer block for each prompt $x_i$ in model $M$ as $c_{-1}(x_i, M)$.  The weights of model $M$'s vocabulary linear layer are represented by $\text{Vocal}(M)$.
We used a series of Baichuan-33B models $\{ M_j | j =1,2,\ldots,n\}$ across Pretraining, SFT, and RLHF stages. An average output $c_{-1}(x_i) = \frac{1}{n} \sum_{j=1}^n c_{-1}(x_i, M_j)$ is computed. Subsequently, $\text{Vocal}(M_i)$ was adjusted to examine the resulting output distribution through:
\$
d_{M_j}(x_i) = \text{softmax}(c_{-1}(x_i) \cdot \text{Vocal}(M_j)),
\$
where $c_{-1}(x_i)$ acts as a consistent reference for input $x_i$ across all models, enabling the assessment of changes in the vocabulary linear layer via the variation in average KL divergence: $\text{KL}(M_j || M_k) = \frac{1}{m} \sum_{i=1}^{m} \text{KL}(d_{M_j}(x_i) || d_{M_k}(x_i))$. The results are depicted in Figure~\ref{fig:changes in KL divergence}.

% We use HalluQA as our prompt input. The outcomes are presented in Figure~\ref{fig:changes in KL divergence}. By analyzing several models from the Pretraining and SFT phases and measuring the variation in KL divergence per 1M training data units, we observed that the weight adjustments in the vocabulary linear layer during SFT and RLHF stages lead to more pronounced changes in KL divergence compared to the later stages of Pretraining. This suggests that modifications within the vocabulary linear layer during SFT significantly enhance the model's expressive capabilities, bringing them closer to its cognitive capacities. This finding aligns with the noted transformation in the LLM's response patterns post-SFT and RLHF, indicating a significant shift in how the model processes and responds to inputs.
Using HalluQA for prompt input, our findings are depicted in Figure~\ref{fig:changes in KL divergence}. Through analysis of various models spanning Pretraining, SFT and RLHF, and by tracking KL divergence changes per 1M training data increments, we found that adjustments to the vocabulary linear layer weights during SFT and RLHF result in more significant KL divergence fluctuations than those observed during later Pretraining stages. This implies that the SFT phase, in particular, notably bolsters the model's expressive capabilities, thus narrowing the gap with its cognitive abilities. This observation is consistent with the apparent shift in LLM response behaviors following SFT and RLHF, suggesting a fundamental shift in the model's token response generation.

\begin{figure}[t]
\begin{center}
\centerline{\includegraphics[width=\columnwidth]{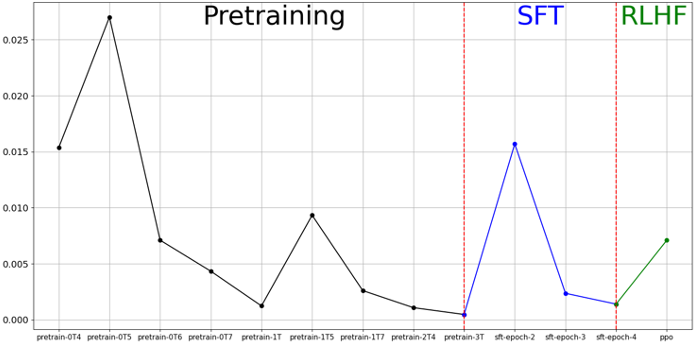}}
\caption{Dynamics of KL divergence per 1M training data, utilizing different vocabulary linear layers, evaluated by HalluQA. }
\vspace{-25pt} % 减少标题下方的空间
\label{fig:changes in KL divergence}
\end{center}
\end{figure}

\section{Methods for Bridging the Gap}
\label{sec:Methods for Bridging the Gap}

\begin{figure}[t]
\begin{center}
\centerline{\includegraphics[width=\columnwidth]{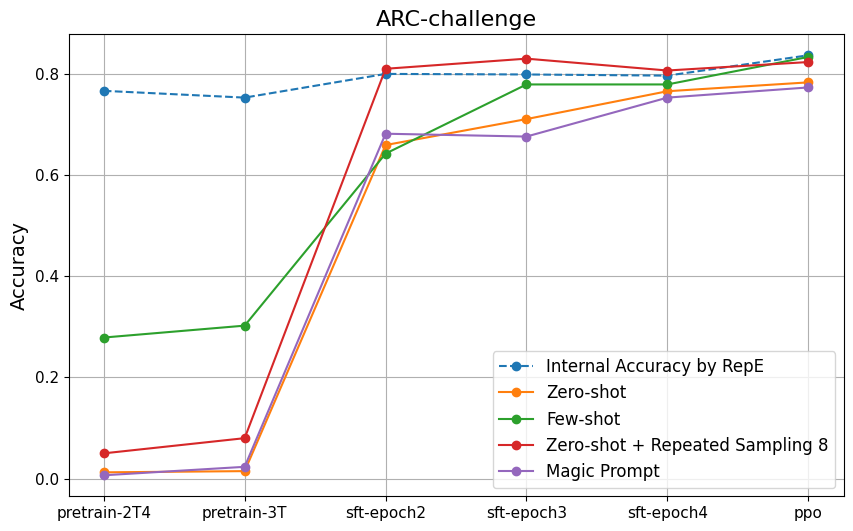}}
\caption{
%Performance for multiple optimization-free methods for bridging the gap between the expressive capability and the cognitive capability.
Performance of the optimization-free methods aimed at bridging the gap between expressive and cognitive capabilities. 
}
\label{fig:bridging the gap between the expressive capability and the cognitive capability}
\vspace{-30pt} % 减少标题下方的空间
\end{center}
\end{figure}

% In this section, we examine some optimization-free methods to test their performance on bridging the gap between the cognitive capability and the expressive capability.
In this section, we explore a selection of optimization-free approaches to evaluate their efficacy in narrowing the gap between cognitive and expressive capabilities.

\subsection{Few-shot Learning}
\label{subsec: few-shot learning}
In few-shot learning, we leverage a templated approach as delineated in Appendix~\ref{subsec: prompt few-shot}. This involves prefacing the input with a series of question-and-answer examples, thereby providing the model with a context that enhances its ability to understand and respond to new, similar queries. The underlying premise of this approach is that providing even a small set of examples can substantially enhance the model's capacity to draw upon its inherent knowledge, thereby amplifying its expressive capabilities.

\subsection{Repeated Sampling}
Building upon the concept of Rejection Sampling as outlined in the work of~\cite{touvron2023llama}, we explore the potential of repeated sampling as a means to extract potentially accurate responses. By generating multiple responses to a single question and considering the response set successful if at least one meets the reference answer, We can explore the upper limit of the model’s expressive capability. Specifically, we sample 8 responses per prompt independently, adopting the criterion that the model's output is deemed accurate if any of these responses is correct. 

\subsection{Magical Additional Prompt}
In light of findings from prior research~\cite{kojima2022large, yang2023large}, it has been suggested that the inclusion of specific, strategically crafted prompts—either preceding or following the primary query—can significantly augment an LLM's performance. To this end, we introduce what we term \emph{magical additional prompts} to our input:

\emph{Let's think step by step and take a deep breathe, the task is very important for human society!}

\subsection{Results Analysis}
The outcomes of our experiments, as depicted in Figure~\ref{fig:bridging the gap between the expressive capability and the cognitive capability}, indicate that both few-shot learning and repeated sampling methodologies exhibit considerable promise in amplifying the expressive capabilities of LLMs. However, the impact of the 'magical additional prompt' was found to be less pronounced. Notably, with repeated sampling employed up to eight times, the LLM's expressive capability is observed to match its cognitive capability. These findings suggest that through strategic prompt engineering, the model can achieve performance levels that surpass its expressive capabilities, hinting that the cognitive capabilities of the model implicitly set the upper bound for its expressive performance.

\section{Conclusion}
\label{sec:conclusion}

% In this paper, we explore the concepts of cognitive and expressive capabilities within LLMs, with a focus on Baichuan-7B and Baichuan-33B, a leading series of bilingual LLMs for Chinese and English. We define cognitive capability through linear representations in the hidden space and assess expressive capability by analyzing the models' direct token outputs. Our comprehensive experiments investigate abilities such as reasoning, common-sense understanding, and logical deduction. Results indicate that the cognitive capability is predominantly established during the Pretraining phase, whereas the expressive capability is then enhanced during the SFT and RLHF. Through hypothesis testing, we demonstrate a significant correlation between the two capabilities. From a theoretical perspective, the paper explores the theoretical underpinnings of the gap, pointing out that the hidden representational space $\cR^m$ is more adept than the token-level space $\cT$ at discerning correct answers from incorrect ones.  We test several widely adopted approaches that might help bridge the gap. Our findings indicate that techniques such as repeated sampling and few-shot learning are particularly effective in enhancing the model's expressive capability, thereby bringing it closer to its cognitive potential.

In this work, we delved into the distinctions between cognitive and expressive capabilities in LLMs, specifically focusing on the bilingual Baichuan-7B and Baichuan-33B series. Cognitive capability is quantified via linear representations in the hidden space, while expressive capability is evaluated through direct token outputs. Our extensive experimentation, encompassing reasoning, common-sense comprehension, and logical inference, reveals that cognitive capabilities are primarily developed during Pretraining, with expressive capabilities further refined in subsequent SFT and RLHF phases. Statistical analysis confirms a strong correlation between these capabilities. Theoretically, we attribute the capability gap to the superior linear separability of the hidden space $\cR^m$ over the token-level space $\cT$. Examination of various optimization-free strategies for mitigating this gap shows that methods like repeated sampling and few-shot learning significantly improve expressive capabilities, aligning them more closely with cognitive capacities. Having established a correlation between linear spaces and the cognitive capabilities of language models, the extraction and the transformation of features within these linear spaces becomes a compelling avenue for future research. 
% \section{Acknowledgement}
% \label{sec:acknowledgement}

\section*{Acknowledgements}
The authors would like to thank Mingyu Xu for the valuable input concerning the effect of residual connection and pre-Layer Normalization. The authors also would like to thank Yiyu Li, Luyao Ma, Lin Bo and Yulong Wang from Baichuan AI for helpful discussions and feedback. 

\section*{Impact Statement}

This paper presents work whose goal is to advance the field of 
Machine Learning. There are many potential societal consequences 
of our work, none which we feel must be specifically highlighted here.

% % In the unusual situation where you want a paper to appear in the
% % references without citing it in the main text, use \nocite
% \nocite{langley00}

\bibliography{example_paper}
\bibliographystyle{icml2024}

%%%%%%%%%%%%%%%%%%%%%%%%%%%%%%%%%%%%%%%%%%%%%%%%%%%%%%%%%%%%%%%%%%%%%%%%%%%%%%%
%%%%%%%%%%%%%%%%%%%%%%%%%%%%%%%%%%%%%%%%%%%%%%%%%%%%%%%%%%%%%%%%%%%%%%%%%%%%%%%
\newpage
\appendix
\onecolumn

\section{Prompt Template}

\subsection{Wrapper prompt for the quantification of cognitive and expressive capabilities}
\label{appen:sec:template}

To quantify the cognitive capability and the expressive capability, we use the following templates as the wrapper for each $(<\text{QUESTION}>, <\text{CHOICE}_i>)$ pair:

\textbf{Template A}

\emph{Consider the correctness of the answer to the following question. Question: }\text{$<\text{QUESTION}>$}\emph{, Answer: }$<\text{CHOICE}>$.\emph{ Directly answer correct or wrong:}

\textbf{Template B}

$<\text{QUESTION}>$ \emph{1.}$<\text{CHOICE}_1>$\emph{ 2.}$<\text{CHOICE}_2>$ \emph{3.}$<\text{CHOICE}_3>$\emph{ 4.}$<\text{CHOICE}_4>$. \emph{Choose only one answer directly.}

\subsection{Additional example prompt for few-shot learning}
\label{subsec: prompt few-shot}
For each $(<\text{QUESTION}>, <\text{CHOICE}_i>)$ pair, we select 2-shot exclusive examples shown as follows:

\begin{emphenv}
Which statement best explains why photosynthesis is the foundation of most food webs?

A. Sunlight is the source of energy for nearly all ecosystems.

B. Most ecosystems are found on land instead of in water.

C. Carbon dioxide is more available than other gases.

D. The producers in all ecosystems are plants.

Choose only one answer directly. A. Sunlight is the source of energy for nearly all ecosystems.

\medskip % 创建一些额外空间
Which piece of safety equipment is used to keep mold spores from entering the respiratory system?

A. safety goggles

B. breathing mask

C. rubber gloves

D. lead apron

Choose only one answer directly. B. breathing mask.

\medskip % 创建一些额外空间
\end{emphenv}

$<\text{QUESTION}>$ 

\emph{A.}$<\text{CHOICE}_1>$ 

\emph{B.}$<\text{CHOICE}_2>$ 

\emph{C.}$<\text{CHOICE}_3>$ 

\emph{D.}$<\text{CHOICE}_4>$.
\begin{emphenv}
Choose only one answer directly.
\end{emphenv}

\section{Detailed Experimental Settings}

\begin{table}[h]
\caption{Size for each datasets.}
\label{supple:table:Size for each datasets}
\begin{center}
\begin{small}
\begin{sc}
\begin{tabular}{lll}
\toprule
Dataset Name & trainset size & testset size\\
\midrule
ARC-easy    &  2251 & 2376 \\
ARC-challenge   &  1119 & 1172 \\
CSQA    &  	9741 & 1140 \\
OBQA   &  4957 & 500 \\
RACE    &  62445 & 	3498 \\
HalluQA   &  240 & 63 \\
\bottomrule
\end{tabular}
\end{sc}
\end{small}
\end{center}
\vskip -0.1in
\end{table}

\begin{table}[h]
\caption{Direct token generation hyperparameters}
\label{table: Direct token generation parameters}
\begin{center}
\begin{small}
\begin{sc}
\begin{tabular}{ll}
\toprule
Term & parameter\\
\midrule
temperature    &  1.2 \\
top p    & 0.9 \\
top k & 50\\
max tokens & 2048\\
repetition penalty & 1.05\\
\bottomrule
\end{tabular}
\end{sc}
\end{small}
\end{center}
\vskip -0.1in
\end{table}

\section{Further Discussion about the Cognition Establishment Process}
\label{supple:sec:Further Discussion about the Cognition Establishment Process}
As is shown in Figure~\ref{fig:the establishment of cognitive capability}, the establishment process of layer-wise cognitive capability is divided into two periods, the bell curve (in the early stages of Pretraining, with less training data) and the plateau curve (in the late stages of Pretraining, with more training data). In the plateau period, the cognitive capability reaches its peak at a certain intermediate layer. Based on this, we believe that cognitive capability may be established in the first few layers of the model, while the pleatue layers continuously strengthen this cognitive capability, and are mapped to expressive capability in the final linear layer. This phenomenon may result largely from \emph{Residual Connection} and \emph{pre-Layer Normalization} in the model architecture. Here is our deduction process.\footnote{This section is contributed by Mingyu Xu and readers can refer to~\cite{men2024shortgpt} for more detail.}

\begin{lemma}
Suppose that $x$ is with components that are independent and identically distributed with a zero mean. The magnitude of the $x$ is directly proportional to its standard deviation: $$\Vert x \Vert = \sqrt{d} \cdot std(x)$$.
\end{lemma}

\begin{lemma}
The variance of the sum of independent variables equals the sum of the variances. $$Var(x+y)=Var(x)+Var(y)$$
\end{lemma}

\begin{lemma}
The relationship between the magnitude of the output $x_L$ and the depth of layer $L$ is as follow: $$\Vert x_L \Vert \sim O(\sqrt{L})$$
\end{lemma}

\begin{proof}
Assuming that the input $x$ and output $y = f(x)$ of each layer are independent, then we know from Lemma 2:
$$Var(x_{L+1})=Var(x_L)+Var(f_L(LN(x_L)))$$
The variance of each layer will increase. From Lemma 1, it is known that the magnitude of each layer will increase. Suppose that $f_L(LN(x_i))$ has a mean value of 0 and a variance of $k$ for each $i$, we have:
$$\Vert x_L \Vert = \sqrt{kLd}$$

The result that the magnitude of the output from each layer increases is experimentally verified on most LLMs such as Baichuan-7B, Baichuan-33B and LLaMa-7B.
\end{proof}

\begin{claim}
\label{supple:claim1}
$LN(x_L)$ and $LN(x_{L+1})$ will become similar as the number of layers $L$ increases.
\end{claim}

Although the absolute value of $f_L(LN(x_L))$ of each layer $L$ may not be small, the relative size compared to $x_L$ will be small. Specially we have:

$$LN(x_L)^T LN(x_{L+1}) = \frac{x_L^T \cdot x_{L+1}}{\Vert x_L \Vert \Vert x_{L+1} \Vert} =\frac{x_L^T(x_L + f_L(LN(x_L)))}{\Vert x_L \Vert \Vert x_{L+1} \Vert} =\frac{\Vert x_L \Vert}{\Vert x_{L+1} \Vert}+\frac{x_L^T f_L(LN(x_L))}{\Vert x_L \Vert \Vert x_{L+1} \Vert}$$

$$=\sqrt{\frac{L}{L+1}}+\frac{x_L^T f_L(LN(x_L))}{\Vert x_L \Vert \Vert x_{L+1} \Vert} \geq \sqrt{\frac{L}{L+1}} - \frac{\Vert x_L \Vert \Vert f_L(LN(x_L)) \Vert}{\Vert x_L \Vert \Vert x_{L+1} \Vert} = \sqrt{\frac{L}{L+1}} - \sqrt{\frac{kd}{k(L+1)d}}$$

$$=\sqrt{\frac{L}{L+1}}-\sqrt{\frac{1}{L+1}}$$

This shows that as the number of layers $L$ increases, the angle between $LN(x_L)$ and $LN(x_{L+1})$ will tend to 0, and the cosine distance will tend to 0, which further shows that these two are very close.

\begin{claim}
\label{supple:claim2}
When $L$ is large, the gradients of two adjacent layers are also very close as well. $$\frac{\partial J}{\partial x_L} \approx \frac{\partial J}{\partial x_{L+1}}$$
\end{claim}

We have:

$$x_{L+1}=x_L + f(LN(x_L))$$

then,

$$
\frac{\partial J}{\partial x_L} = \frac{\partial J}{\partial x_{L+1}} + \frac{\partial J}{\partial x_{L+1}} \frac{\partial f(LN(x_L))}{\partial x_{L}} = \frac{\partial J}{\partial x_{L+1}} + \frac{\partial J}{\partial x_{L+1}} \frac{\partial f(LN(x_L))}{\partial LN(x_L)} \frac{\partial LN(x_L)}{\partial x_L}
$$

We observe that $\frac{\partial f(LN(x_L))}{\partial LN(x_L)}$ will not be big. For example, if $f$ is a linear, then $\frac{\partial f(LN(x_L))}{\partial LN(x_L)}=W^T$, and $\Vert W \Vert$ cannot be big because of weight decay. Another term $\frac{\partial LN(x_L)}{\partial x_L} = O(\frac{1}{\Vert x_L \Vert})$, with a big $L$, this term is going to be small (according to Claim 3). So that

$$\frac{\partial J}{\partial x_L} = \frac{\partial J}{\partial x_{L+1}}(1+O(\frac{1}{\sqrt{L}}))\approx \frac{\partial J}{\partial x_{L+1}}$$

Integrating Claim~\ref{supple:claim1} and \ref{supple:claim2}, we understand that when $L$ is large, the inputs to layers $L$ and $L+1$ are very similar, as are the gradients propagated back through them, and the network architecture remains identical. Thus, these two layers will be extremely alike. The contextual information accessible to layer $L$ is essentially accessible to layer $L+1$ as well. The only difference is a slight reduction in the number of attention cycles, which may lead to a decrease in the intensity of attention.

In summary, the \emph{Residual Connections} lead to an increasing norm of the residual branches, and the addition of \emph{pre-Layer Normalization} results in very similar inputs for the Attention and MLP layers between adjacent layers in deeper networks. On the other hand, the increasing norm of the \emph{Residual Connections} branches and the presence of \emph{pre-Layer Normalization} also cause that the errors propagated back through adjacent layers in deeper networks are very similar. These factors contribute to a high degree of similarity between layers $L+1$ and $L$ and the plateau in the layer-wise cognition capability measurement. 

These theoretical results and the appearance of the cognitive capability plateau curve in the Pretraining stage (Figure~\ref{fig:the establishment of cognitive capability}) may represent some redundancy of the model. To see this, we did several additional experiments.

\paragraph{Experiment 1.} We delete the 23-th layer of Baichuan-7B (31 layers in total) and connect the rest of the model directly, the performance on MMLU and CMMLU almost doesn't decline, which aligns with the result on the layer-wise cognition capability curve.

\begin{table}[h]
\caption{Deleting one redundant intermediate layer.}
\label{supple:table:Deleting one redundant intermediate layer}
\begin{center}
\begin{small}
\begin{sc} 
\begin{tabular}{lcc}
\toprule
Dataset Name & Baichuan-7B & Baichuan-7B (delete 23-th layer)\\
\midrule
MMLU    &  0.5416 & 0.5398 \\
CMMLU   &  0.5707 & 0.5659 \\
\bottomrule
\end{tabular}
\end{sc}
\end{small}
\end{center}
\vskip -0.1in
\end{table}

\paragraph{Experiment 2.} We found that basically speaking, the larger/deeper the LLM, the more redundancy is likely to be. This suggests that under the current training data volume, the model size has a lot of room for optimization. To see this, we conduct the cognition capability measurement on Phi-2, a 2.7B "small" LLM that is reported to have the on-par capability with LLaMa2-7B and LLaMa2-13B. The plateau curve is also observed but the plateau part is much shorter. The results are shown as follow:

\begin{table}[h]
\caption{The relationship between the plateau propotion and the model size.}
\label{supple:table:The relationship between the plateau propotion and the model depth}
\begin{center}
\begin{small}
\begin{sc}
\begin{tabular}{lll}
\toprule
Model Name & Model Depth & "Plateau" propotion\\
\midrule
Phi-2 & 32 & 0.225 \\
Baichuan2-7B & 32 & 0.53125 \\
LLaMa2-7B & 32 & 0.53125 \\
Baichuan2-13B & 40 & 0.575 \\
LLaMa2-13B & 40 & 0.575 \\
Baichuan2-33B & 60 & 0.55 \\
Baichuan2-53B & 64 & 0.5873 \\
\bottomrule
\end{tabular}
\end{sc}
\end{small}
\end{center}
\vskip -0.1in
\end{table}

Some recent works regarding model merging (mainly describes the hard structural merging of two different models without reducing model capabilities) also suggest the same phenomenon~\cite{yadav2024ties,goddard2024arcee}. One conjecture is that these layers in plateau may be redundant in terms of cognitive capability, but play a role in establishing expressive capability. We observe that in the tests in MMLU and CMMLU (experiment 1), if one of the layers is deleted, the logits values of the four options A, B, C, and D will become smaller. It is inferred that the $L+1$ layer only strengthens the prediction results of the original $L$ layer, and help the model to better output token-level results. 

\section{Supplementary Experimental Results}

\subsection{Hypothesis testing}
\label{appen:subsec:hypothesis testing}

We provide the supplementary results of hypothesis testing in ARC-Challenge, ARC-Easy, CommonSenseQA, OpenbookQA and RACE, in addition to the results shown in the main part. The results are shown in Table~\ref{appen: table: hypothesis testing probability}.

\begin{table*}[ht]
\caption{The statistical correlation analysis of the cognitive capability and the expressive capability. H-Test probability stands for hypothesis testing probability with the hypothesis that they are irrelevant. We refer the expressive accuracy and the cognitive accuracy to the accuracy that are obtained by the quantification of the expressive capability and the cognitive capability as mentioned in Section~\ref{sec: main results}. }
\label{appen: table: hypothesis testing probability}
\vskip 0.15in
\begin{center}
\begin{small}
\begin{sc}
\begin{tabular}{p{2.5cm}lcM{1.8cm}M{1.8cm}M{2.1cm}M{2cm}}
\toprule
Data set & Model & Consistency & Expressive Accuracy & Cognitive Accuracy & Binomial Distribution & H-Test probability\\
\midrule
ARC Challenge    & SFT Epoch2& 75.00\% & 65.88\% & 78.59\% & $\cB(300, 0.5908)$ & 0.23\%\\
                 & SFT Epoch3& 77.00\% & 71.01\% & 78.59\% & $\cB(300, 0.6201)$ & $<$0.01\%\\
                 & SFT Epoch4& 77.92\% & 76.51\% & 78.92\% & $\cB(300, 0.6533)$ & $<$0.01\%\\
                 & PPO       & 80.20\% & 78.26\% & 82.60\% & $\cB(300, 0.6842)$ & 0.01\%\\
\midrule
ARC Easy         & SFT Epoch2& 79.43\% & 81.13\% & 92.40\% & $\cB(570, 0.7639)$& 1.04\%\\
                 & SFT Epoch3& 89.28\% & 84.01\% & 92.04\% & $\cB(570, 0.7859)$& $<$0.01\%\\
                 & SFT Epoch4& 89.10\% & 88.22\% & 92.98\% & $\cB(570, 0.8285)$& 0.01\%\\
                 & PPO       & 89.34\% & 87.92\% & 92.63\% & $\cB(570, 0.8233)$& $<$0.01\%\\
\midrule
CommonSenseQA    & SFT Epoch2& 60.28\% & 54.60\% & 76.98\% & $\cB(1221, 0.5248)$& $<$0.01\%\\
                 & SFT Epoch3& 64.70\% & 55.62\% & 76.26\% & $\cB(1221, 0.5295)$& $<$0.01\%\\
                 & SFT Epoch4& 66.83\% & 65.98\% & 77.39\% & $\cB(1221, 0.5870)$& $<$0.01\%\\
                 & PPO       & 68.14\% & 66.75\% & 77.23\% & $\cB(1221, 0.5917)$& $<$0.01\%\\
\midrule
OpenBookQA       & SFT Epoch2& 62.14\% & 53.51\% & 75.20\% & $\cB(500, 0.5176)$& $<$0.01\%\\
                 & SFT Epoch3& 70.28\% & 60.00\% & 74.80\% & $\cB(500, 0.5496)$& $<$0.01\%\\
                 & SFT Epoch4& 71.74\% & 63.53\% & 74.00\% & $\cB(500, 0.5649)$& $<$0.01\%\\
                 & PPO       & 68.33\% & 61.52\% & 73.20\% & $\cB(500, 0.5534)$& $<$0.01\%\\
\midrule
RACE             & SFT Epoch2& 72.26\% & 68.72\% & 79.42\% & $\cB(3451, 0.6101)$& $<$0.01\%\\
                 & SFT Epoch3& 73.24\% & 73.17\% & 78.94\% & $\cB(3451, 0.6341)$& $<$0.01\%\\
                 & SFT Epoch4& 74.24\% & 78.94\% & 74.00\% & $\cB(3451, 0.6563)$& $<$0.01\%\\
                 & PPO       & 73.92\% & 77.01\% & 80.10\% & $\cB(3451, 0.6410)$& $<$0.01\%\\
\bottomrule
\end{tabular}
\end{sc}
\end{small}
\end{center}
\vskip -0.1in
\end{table*}

\subsection{Methods for bridging the gap}
\label{appen:subsec:methods for bridging the gap}

We provide the supplementary results of methods for bridging the gap between cognitive and expressive capabilities in ARC-Challenge, ARC-Easy, CommonSenseQA, OpenbookQA and RACE, in addition to the results shown in the main part. The results are shown in Table~\ref{appen:fig:gap between expressive and cognitive capability}.

\begin{figure*}[h]
	\centering
	\begin{subfigure}
		\centering
		\includegraphics[width=0.33\columnwidth]{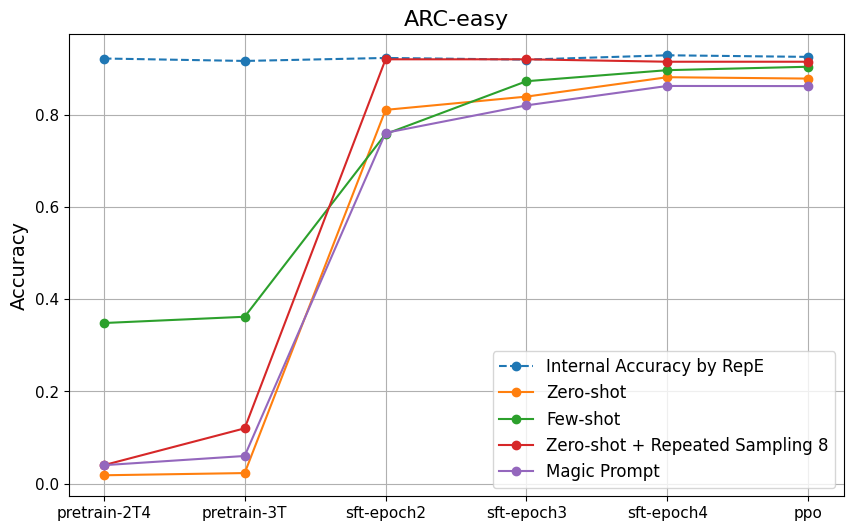}
	\end{subfigure}
	\centering
	\begin{subfigure}
		\centering
		\includegraphics[width=0.32\columnwidth]{figure/baseline_arc_challenge.png}
	\end{subfigure}
	\centering
 	\begin{subfigure}
		\centering
		\includegraphics[width=0.32\columnwidth]{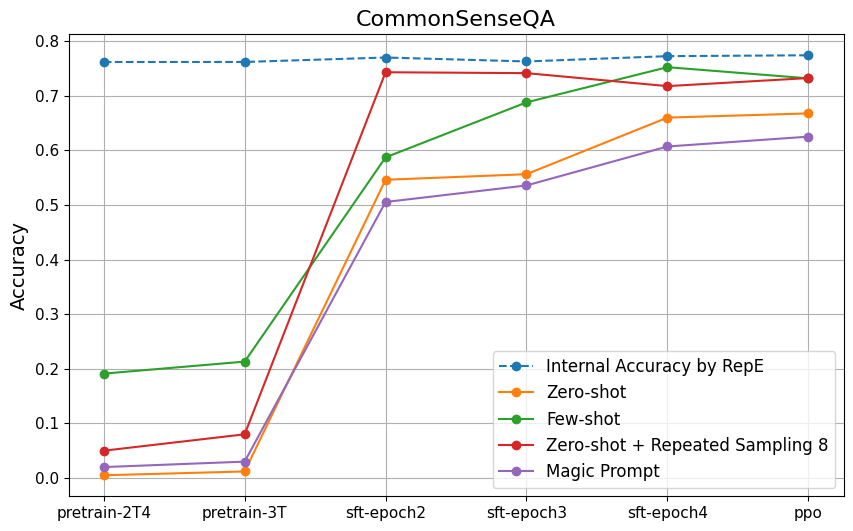}
	\end{subfigure}
        \centering
        \begin{subfigure}
		\centering
		\includegraphics[width=0.32\columnwidth]{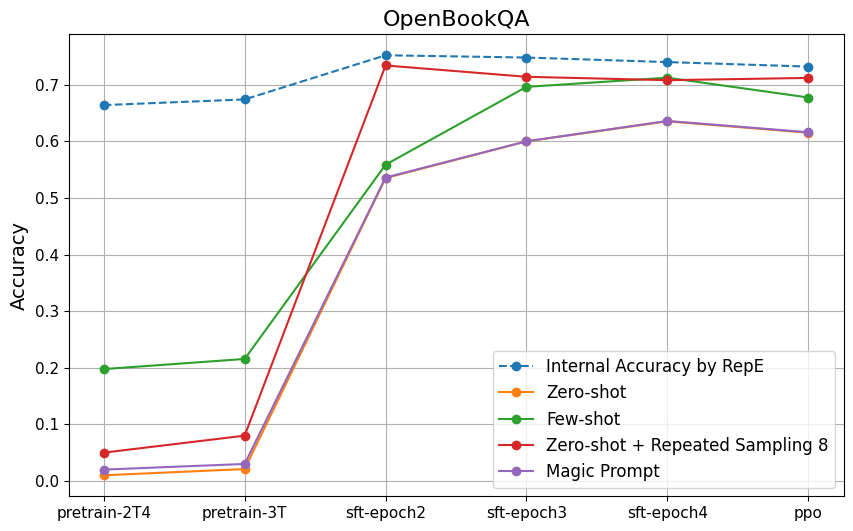}
	\end{subfigure}
         \centering
        \begin{subfigure}
		\centering
		\includegraphics[width=0.32\columnwidth]{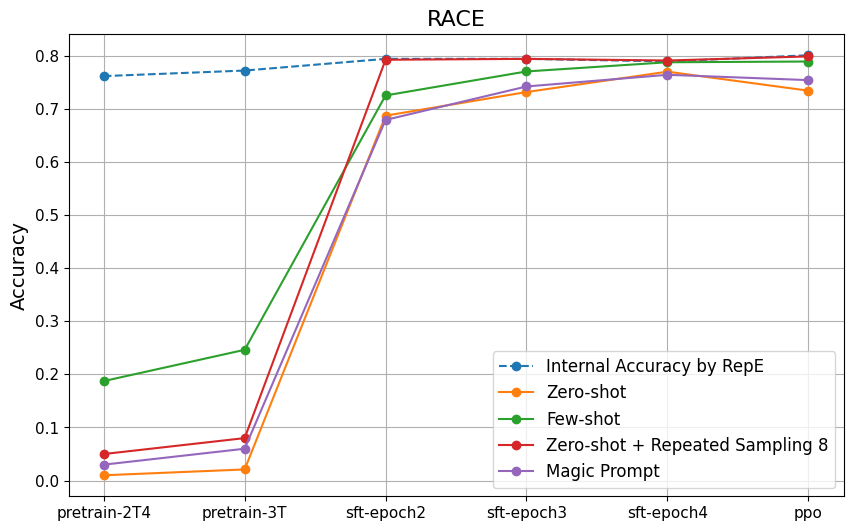}
	\end{subfigure}
	\caption{Methods for bridging the gap between the expressive capability and the cognitive capability.}
	\label{appen:fig:gap between expressive and cognitive capability}
	\vskip -0.2in
\end{figure*}

\begin{figure*}[t]
	\centering
	\begin{subfigure}
		\centering
		\includegraphics[width=0.49\columnwidth]{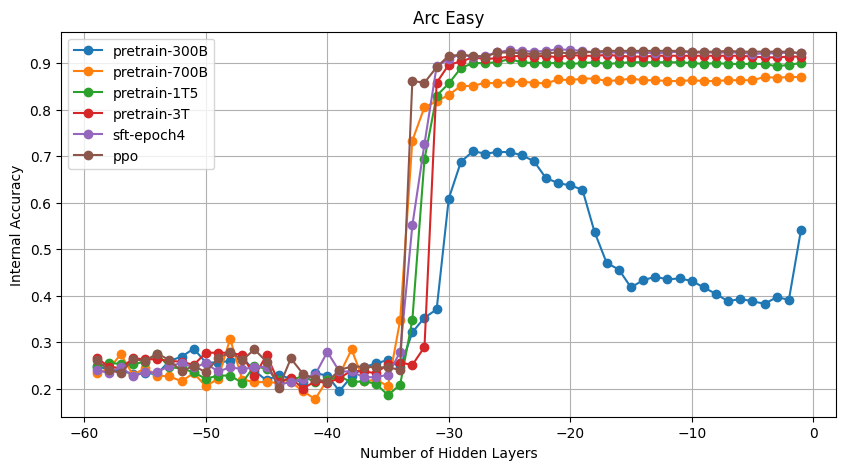}
	\end{subfigure}
	\centering
	\begin{subfigure}
		\centering
		\includegraphics[width=0.49\columnwidth]{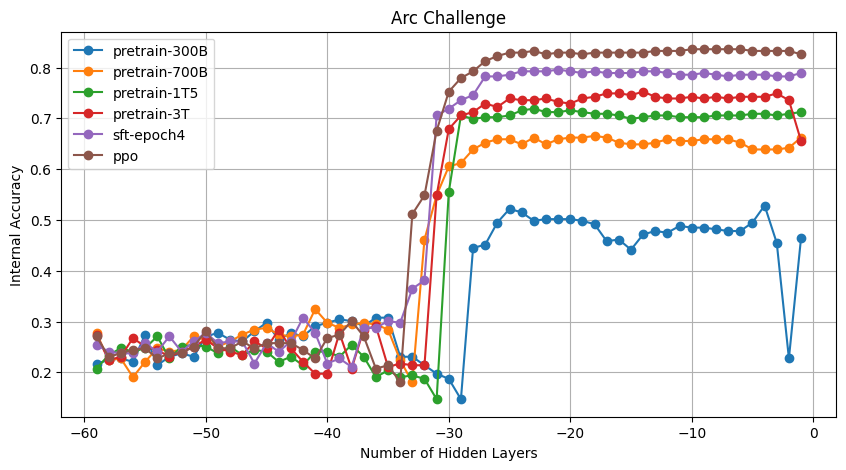}
	\end{subfigure}
	\centering
 	\begin{subfigure}
		\centering
		\includegraphics[width=0.49\columnwidth]{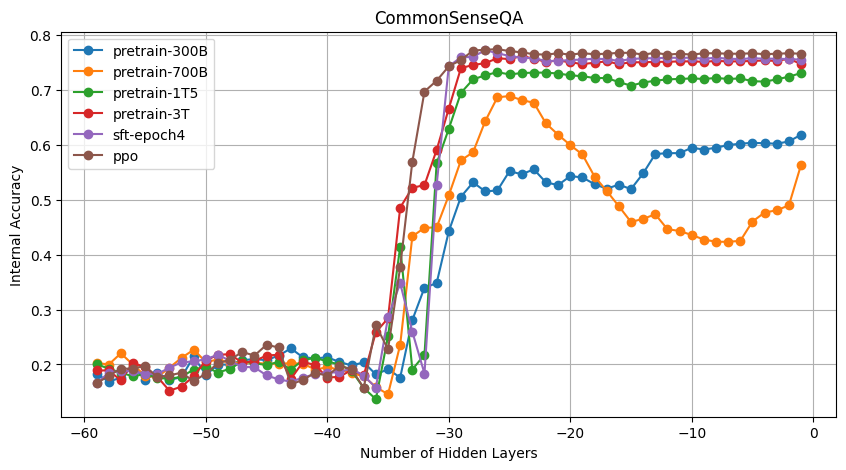}
	\end{subfigure}
        \centering
        \begin{subfigure}
		\centering
		\includegraphics[width=0.49\columnwidth]{figure/pca_layer_obqa.png}
	\end{subfigure}
	\caption{The performance of each layer of representation engineering in Baichuan-33B in different training stage, which reflect the internal establishment process of the cognitive capability.}
	\label{appen:fig:the establishment of cognitive capability}
	\vskip -0.2in
\end{figure*}

\subsection{The establishment of cognitive capability in Pretraining}
\label{appen:subsec:Pretraining process of Baichuan-7B}

We provide the supplementary results of the cognitive capability (measured by Algorithm~\ref{alg: representation engineering}) establishment process in Pretraining phase in both Baichuan-7B and Baichuan-33B, in addition to the results shown in the main part. 

The progression of cognitive capability in Baichuan-7B, assessed by Algorithm~\ref{alg: representation engineering}, is depicted in Figure~\ref{appen: fig: pretrain pca in 7B}.

The performance of the linear representations for each layer in Baichuan-33, which reflect the internal establishment process of the cognitive capability. The results in Baichuan-7B is shown in Figure~\ref{appen:fig:the establishment of cognitive capability for Baichuan-7B} and the results in Baichuan-33B is shown in Figure~\ref{appen:fig:the establishment of cognitive capability}.

\begin{figure}[h]
\begin{center}
\centerline{\includegraphics[width=0.48\columnwidth]{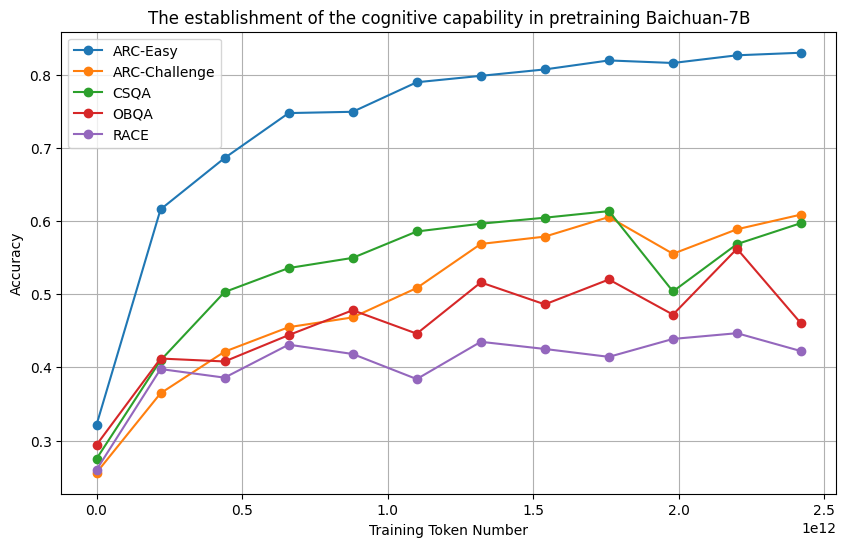}}
\caption{The figure shows the increasing process of the cognitive capability in the Pretraining stage in Baichuan-7B.}
\label{appen: fig: pretrain pca in 7B}
\end{center}
\end{figure}

\begin{figure*}[ht]
	\centering
	\begin{subfigure}
		\centering
		\includegraphics[width=0.48\columnwidth]{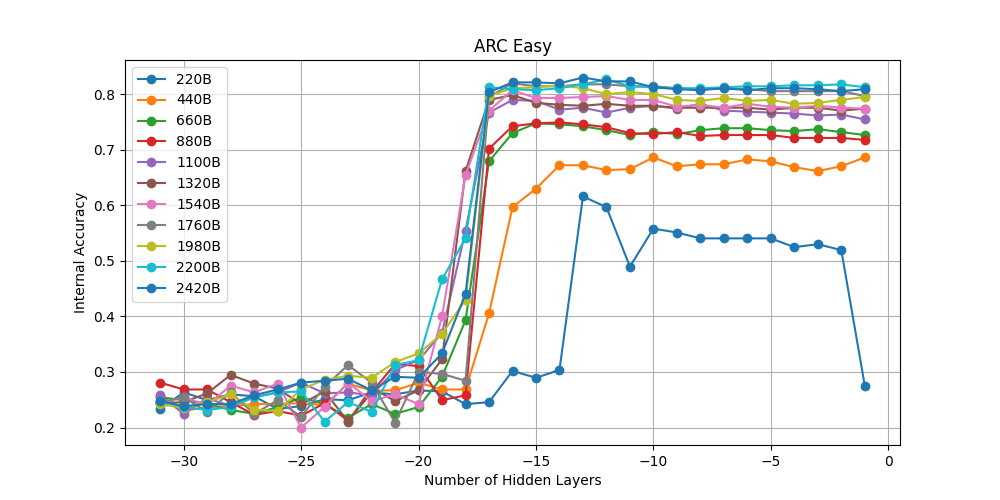}
	\end{subfigure}
	\centering
	\begin{subfigure}
		\centering
		\includegraphics[width=0.48\columnwidth]{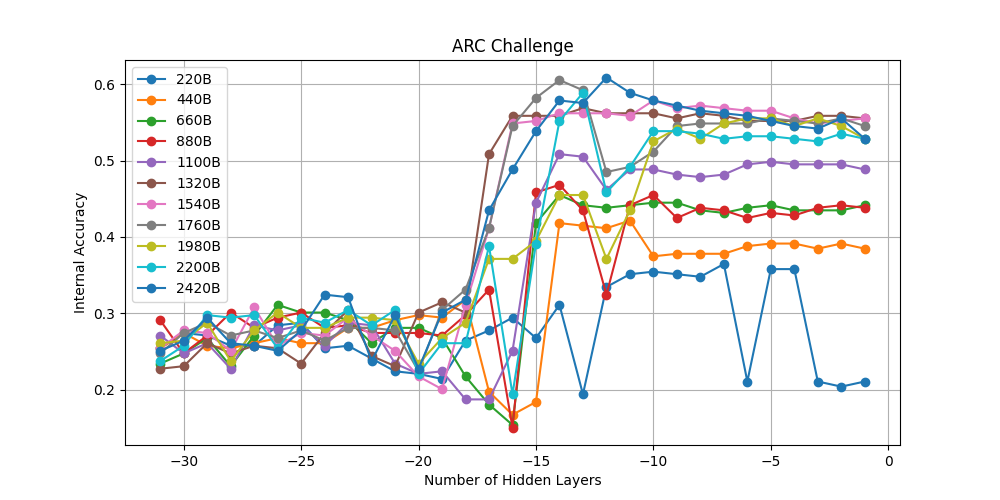}
	\end{subfigure}
	\centering
 	\begin{subfigure}
		\centering
		\includegraphics[width=0.48\columnwidth]{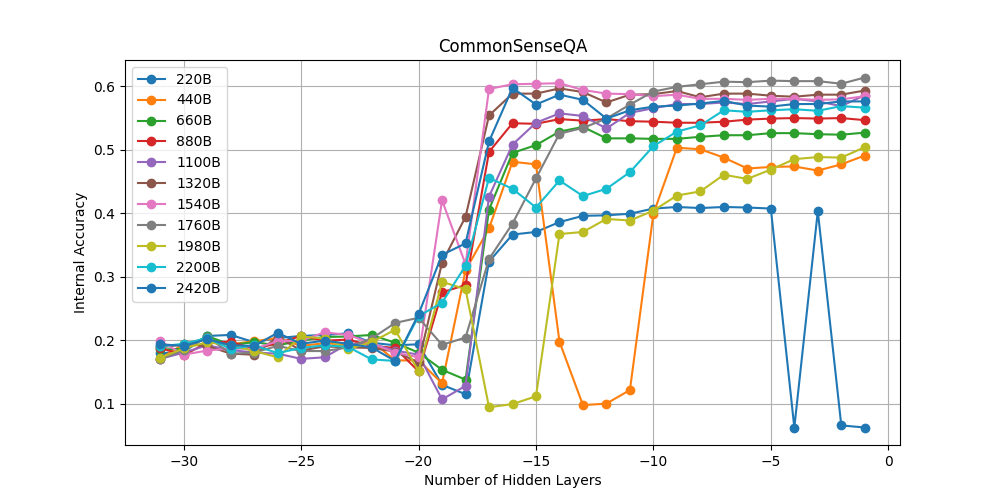}
	\end{subfigure}
        \centering
        \begin{subfigure}
		\centering
		\includegraphics[width=0.48\columnwidth]{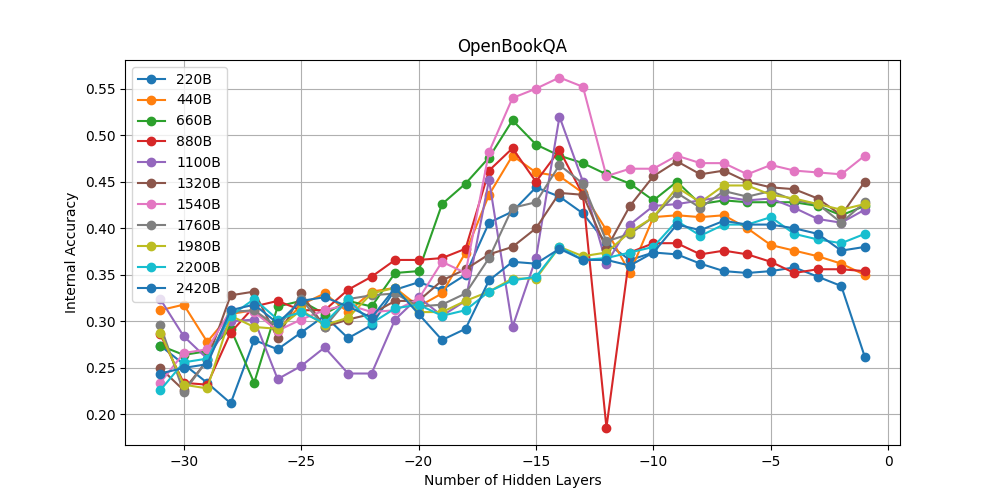}
	\end{subfigure}
	\caption{The performance of each layer of representation engineering in Baichuan-7B in different training stage, which reflect the internal establishment process of the cognitive capability.}
	\label{appen:fig:the establishment of cognitive capability for Baichuan-7B}
	\vskip -0.2in
\end{figure*}

\subsection{The gap between the quantified cognitive capability and the lm-evaluation-harness performance}
\label{appen:subsec:The gap between the quantified cognitive capability and the lm-evaluation-harness performance}

We applied the general lm-evaluation-harness framework~\cite{eval-harness} to both models, which employs greedy search to evaluate each option's probability for answer selection. The outcomes are shown in Figure~\ref{fig:lm evaluation harness in Baichuan-33B} for Baichuan-33B and in Figure~\ref{fig:lm evaluation harness in Baichuan-7B} for Baichuan-7B. 

\begin{figure*}[ht]
	\centering
	\begin{subfigure}
		\centering
		\includegraphics[width=0.48\columnwidth]{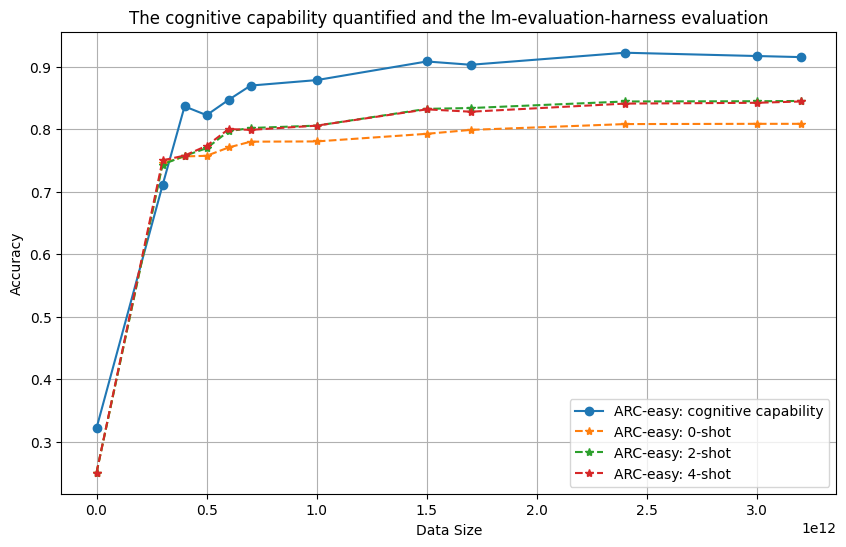}
	\end{subfigure}
	\centering
	\begin{subfigure}
		\centering
		\includegraphics[width=0.48\columnwidth]{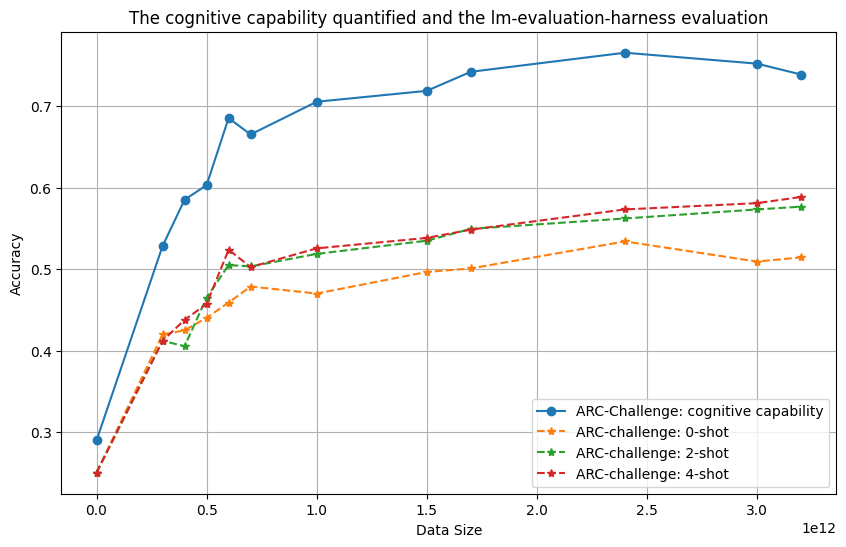}
	\end{subfigure}
	\centering
 	\begin{subfigure}
		\centering
		\includegraphics[width=0.48\columnwidth]{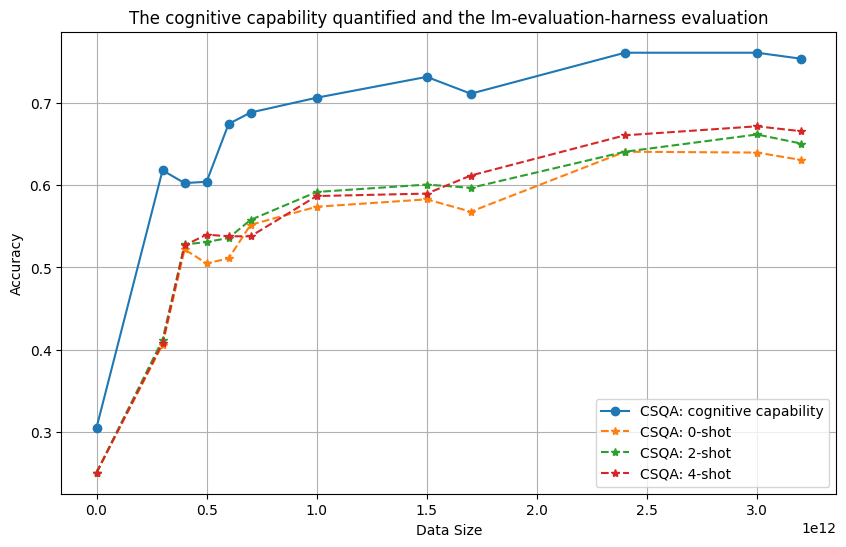}
	\end{subfigure}
        \centering
        \begin{subfigure}
		\centering
		\includegraphics[width=0.48\columnwidth]{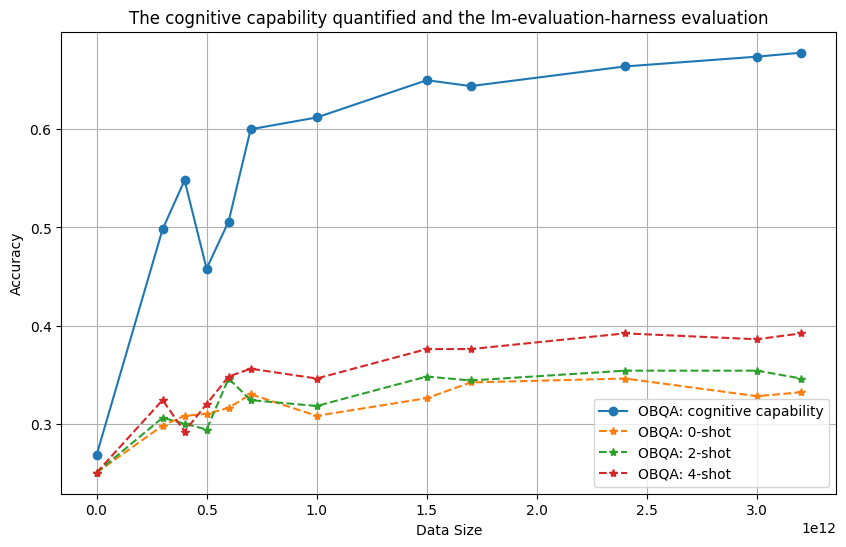}
	\end{subfigure}
	\caption{The gap between the quantified cognitive capability and the lm-evaluation-harness performance in Baichuan-33B.}
	\label{fig:lm evaluation harness in Baichuan-33B}
\end{figure*}

\begin{figure*}[ht]
	\centering
	\begin{subfigure}
		\centering
		\includegraphics[width=0.48\columnwidth]{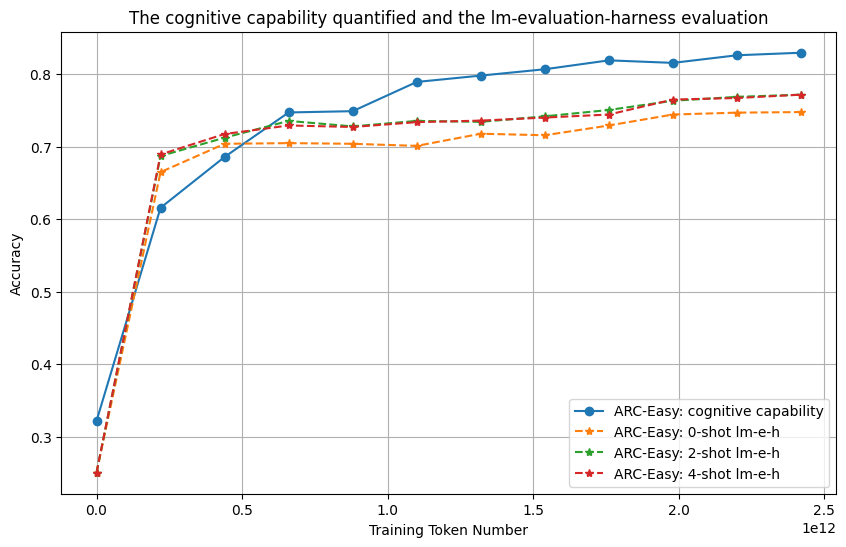}
	\end{subfigure}
	\centering
	\begin{subfigure}
		\centering
		\includegraphics[width=0.48\columnwidth]{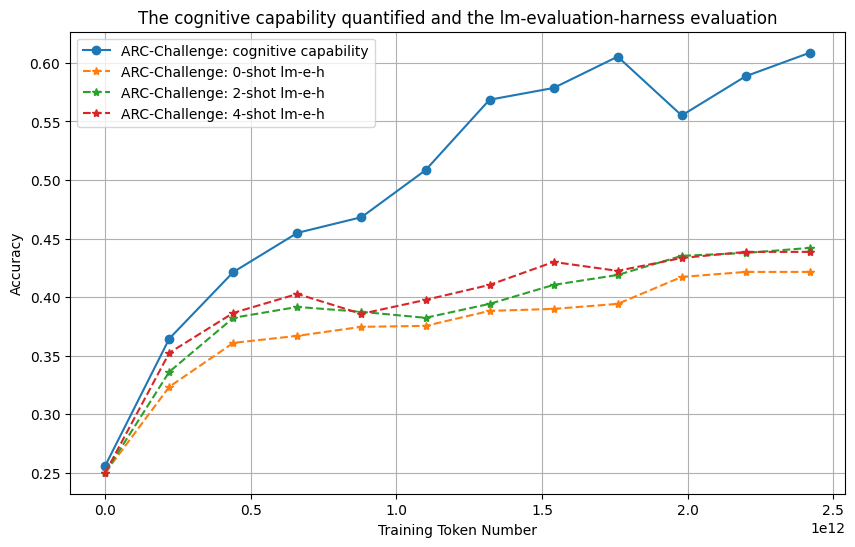}
	\end{subfigure}
	\centering
 	\begin{subfigure}
		\centering
		\includegraphics[width=0.48\columnwidth]{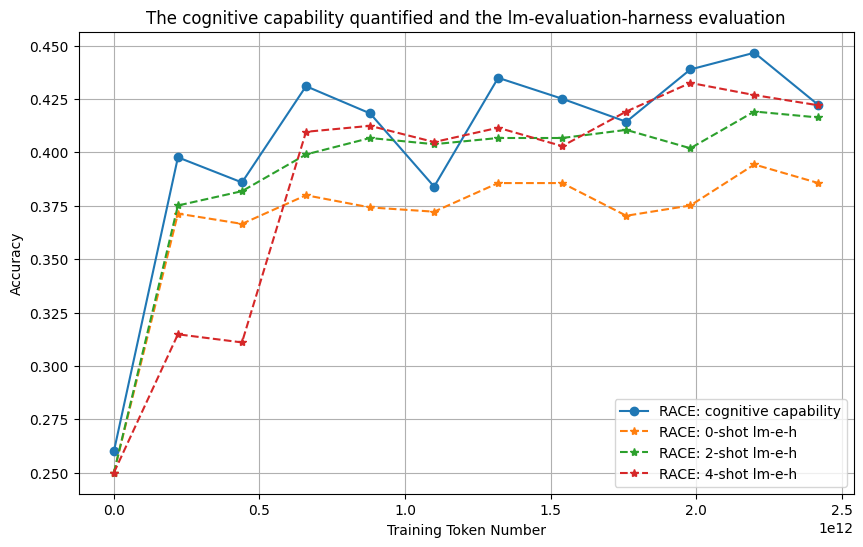}
	\end{subfigure}
        \centering
        \begin{subfigure}
		\centering
		\includegraphics[width=0.48\columnwidth]{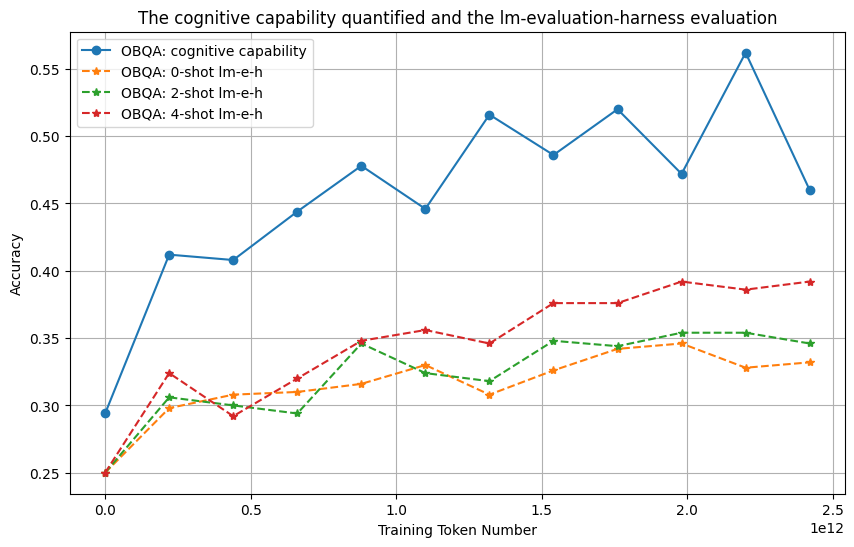}
	\end{subfigure}
	\caption{The gap between the quantified cognitive capability and the lm-evaluation-harness performance in Baichuan-7B.}
	\label{fig:lm evaluation harness in Baichuan-7B}
\end{figure*}

\section{Case Study}
\label{appen:sec:case study}

\begin{figure}[t]
\begin{center}
\centerline{\includegraphics[width=0.48\columnwidth]{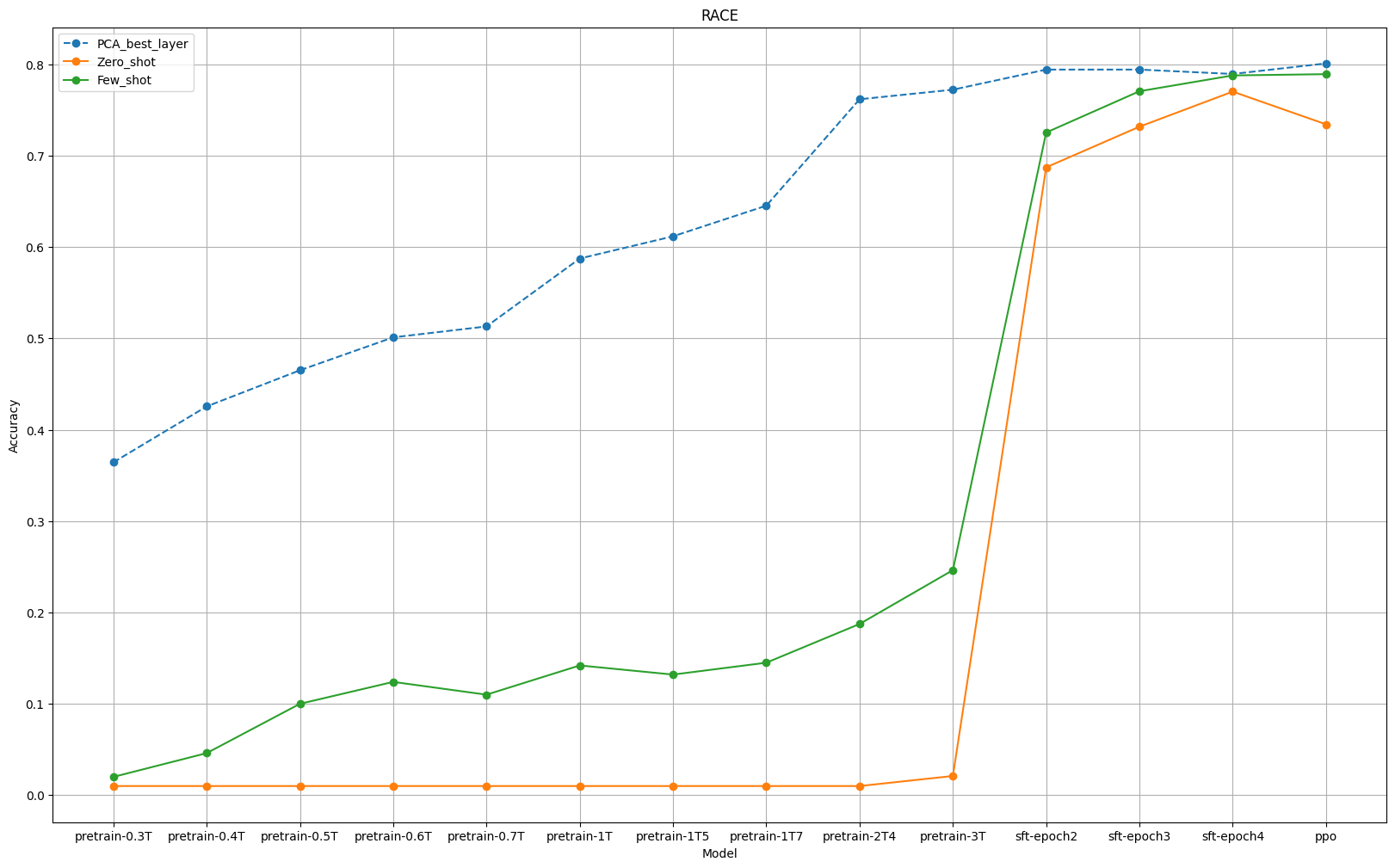}}
\caption{An example in RACE: The complete establishment process of the cognitive capability and the expressive capability in Pretraining, SFT and RLHF. }
\label{fig:gap in race}
\end{center}
\end{figure}

\begin{table}[ht]
\centering
\begin{tabular}{|l|p{12cm}|}
\hline
\textbf{Question} & Which technology was developed most recently?

A. cellular telephone

B. television

C. refrigerator

D. airplane

Choose only one answer directly. \\ 
\hline
\textbf{Reference Answer} & A. cellular telephone \\
\hline
pretrain-2T4 & e choice, click on the mouse and enter.You can make up to five attempts - it is for free.If you made the first choice (not correct), try other possibilities by clicking here: other options)  How to solve this test correctly in order to save time and have successful progress within next tests?To do this you should have access (knowledge of) to all of the given answers. This is not a problem, all of you who want know an answer to a specific choice, just go to the end of my ePortafolio for a full explanation of any choice available. Do not miss the chance to learn, for the sake of saving some time while answering some questionnaires ... You cannot progress or do the work unless you put knowledge before time, at least in this case. Take the opportunity to be really well prepared during this whole survey process! \\
\hline
pretrain-3T & 39) By which means did the Incas carry messages and information? A. telephone B. postal service C. human couriers D. radio wave Choose only one answer directly from the text. 40) To live in caves of rocks is called A nomadic life B agrarian life C sedentary life D fishing and seafaring life 41) Which type of food the hunters eat? A. meat B. meat and blood C. blood, skin and meat D. milk  Choose only one answer from the text. 42) The hunter/gatherer made tools to chop or break food. What tool did they use when cutting food for cooking or serving it on plates, spoons and knives? A. scissors B. spuds C. knives and forks D. knives  Choose only one answer. 43) Today some people live by moving from place to place eating, hunting and collecting roots, berries and other wild plants. What people are they? A. agriculturists B. forest farmers C. hunter/gatherers D. forest dwellers  Choose only one word from the text for the blank. SECTION I B  50 MINUTE TIME  Read the following passages and choose the correct answers according to the given instructions.\\
\hline
SFT-epoch3 & They all appeared within about 1900, but I think television was introduced first (roughly 1925) and then refrigerators in the 1930s, followed by cellular phones in the 1980s. Airplanes were around well before this (as early as the beginning of the 20th century). So choosing only among these options the airplane would be the correct option. \\
\hline
SFT-epoch4 & A. cellular telephone \\
\hline
PPO & A. cellular telephone \\
\hline
\end{tabular}
\caption{A test example: The response to the same question in different training stages for Baichuan-33B.}
\end{table}

\end{document}